\documentclass{article}

\usepackage[main, final]{neurips_2026}
\usepackage{microtype}
\usepackage{hyperref}
\usepackage{url}
\usepackage{booktabs}
\usepackage{microtype}
\usepackage{graphicx}
\usepackage{subcaption}
\usepackage{booktabs} 
\usepackage{booktabs}
\usepackage{multirow}
\usepackage{amsmath}
\usepackage{algorithm}
\usepackage{algorithmic}
\usepackage{tabularx}
\usepackage{booktabs}
\usepackage{xcolor}
\usepackage{enumitem}
\usepackage{mathtools}
\usepackage{wrapfig}
\usepackage{amsthm}
\newtheorem{theorem}{Theorem}[section]

\newtheorem{remark}[theorem]{Remark}
\newtheorem{corollary}[theorem]{Corollary}

\usepackage[utf8]{inputenc} 
\usepackage[T1]{fontenc}    
\usepackage{hyperref}       
\usepackage{url}            
\usepackage{booktabs}       
\usepackage{amsfonts}       
\usepackage{nicefrac}       
\usepackage{microtype}      
\usepackage{xcolor}         

\title{Bilevel Optimization of Synthetic Trajectories for Multi-Turn LLM Fine-Tuning}

\author{%
  Shresth Verma \\
  Harvard University\\
  \texttt{sverma@g.harvard.edu} \\
  \And
  Mauricio Tec\thanks{Work done while a Postdoctoral Fellow at Harvard University.} \\
  Amazon Science\\
  \texttt{mauriciogtec@gmail.com} \\
  \And
  Cheol Woo Kim \\
  Harvard University\\
  \texttt{cwkim@g.harvard.edu} \\
  \AND
  Kai Wang \\
  Georgia Institute of Technology\\
  \texttt{kwang692@gatech.edu} \\
  \And
  Milind Tambe \\
  Harvard University\\
  \texttt{tambe@g.harvard.edu} \\
}

\usepackage{booktabs}
\usepackage{multirow}
\usepackage{graphicx}
\usepackage{array}
\usepackage{xcolor}
\usepackage{colortbl}
\usepackage{geometry}
\usepackage{caption}
\usepackage{mdframed}
\definecolor{tqbg}{RGB}{232,242,252}
\definecolor{carbg}{RGB}{226,246,240}
\definecolor{webbg}{RGB}{251,240,222}
\definecolor{headergray}{RGB}{245,245,245}
 
\newcolumntype{R}[1]{%
  >{\centering\arraybackslash}p{#1}%
}
\begin{document}

\maketitle

\begin{abstract}
While LLMs excel at single-turn generation, they struggle with long-horizon, multi-turn interactions.
Offline reinforcement learning (RL) offers a scalable approach, yet its performance hinges on the availability and quality of multi-turn trajectory data. A common remedy is to augment training with synthetic trajectories generated by LLMs or simulators, but synthetic data is highly heterogeneous in quality, and naively treating all trajectories as equally informative can degrade performance.
We propose BOOST, a bilevel optimization framework where the inner
level trains the LLM on reweighted data and the outer level trains a lightweight
reweighting head on held-out real validation tasks, assigning continuous
trajectory-level weights without requiring an external judge.
To ground this approach, we
derive a PAC-Bayesian bound revealing a three-way trade-off: synthetic data increases
diversity but risks task-shift, while concentrating weight on high-quality trajectories
improves empirical performance at the cost of effective sample size.
Empirically, our method consistently outperforms multiple baselines. Analysis reveals it upweights synthetic trajectories that align with the real data distribution and exhibit higher qualitative merit.

\end{abstract}

\section{Introduction}

Large Language Models (LLMs) excel at generating fluent, human-like text, yet they often struggle with long-horizon reasoning and interaction \citep{liu2024lost,li2025beyond,han2025can,laban2026llms}.
In multi-turn settings such as negotiation, persuasion, or technical troubleshooting, effective behavior requires planning across multiple steps, where intermediate actions may not yield immediate rewards but are essential for long-term success.
For instance, a capable agent might ask a clarifying question rather than providing a premature answer, prioritizing future outcomes over short-term quality \citep{aliannejadi2019asking,kuhn2022clam,abdulhai2025lmrl}.

To address this limitation, multi-turn interaction can be framed as a reinforcement learning (RL) problem. While online RL optimizes for long-term objectives, model scale and interaction costs make it prohibitively expensive for LLMs. Consequently, offline RL has emerged as a scalable alternative, learning policies from pre-collected trajectories \citep{snell2022offline,chebotar2023q,hong2025planning}. Despite its scalability advantages, offline RL introduces a new bottleneck: the availability and quality of multi-turn trajectory data. High-performing offline RL methods require datasets with sufficient coverage of the state–action space, which is especially challenging in multi-turn domains.  To mitigate this issue, recent work has explored augmenting training data with synthetic trajectories generated by LLMs or simulators \citep{singh2023beyond,goldie2025synthetic}. While this improves coverage, it raises a critical challenge: synthetic data is highly heterogeneous in quality. A naive approach is to treat all trajectories as equally informative, but this can lead to suboptimal or even degraded performance, as low-quality or misleading trajectories may dominate the learning signal \citep{ alemohammad2023self,shumailov2023curse,gerstgrasser2024model}. Prior work has attempted to address this via binary filtering using a stronger LLM judge, retaining only trajectories that pass a quality threshold, or using expert knowledge base for generating synthetic data \citep{zelikman2022star,gunasekar2023textbooks,goldie2025synthetic}. While effective, such approaches discard potentially useful signal and rely on access to a higher-capacity scorer.


We propose a principled bilevel optimization framework for incorporating synthetic data into multi-turn RL for LLM fine-tuning. 
Alongside fine-tuning the LLM, we learn a lightweight reweighting head that assigns a scalar weight to each trajectory in a mixed dataset. These normalized weights modulate each trajectory's contribution in standard offline RL objectives, ensuring compatibility with various algorithms. Unlike heuristic filtering, our method performs continuous reweighting based on validation performance. 
We formulate the learning of reweighting model as a bilevel optimization problem: the inner level trains the policy on reweighted data, while the outer level optimizes the weights to improve performance on held-out real validation tasks. 
This formulation is similar to decision-focused learning~\citep{donti2017task,elmachtoub2022smart,wang2023scalable}, which trains predictive models end-to-end for the quality of the downstream decisions they induce. Analogously, BOOST does not treat synthetic data selection as an isolated preprocessing step. Instead, each trajectory's weight is learned according to its effect on the downstream multi-turn decisions of the LLM fine-tuned on it.
Further, prior work on data reweighting focuses on coarse-grained data source reweighting \citep{pan2025scalebio, xie2023doremi}. Our method operates at the trajectory level, enabling fine-grained control over data quality. This is particularly important in settings with synthetic data, where variability in quality can be substantial even within a single data source. Empirically, we demonstrate that our approach improves multi-turn RL performance by effectively leveraging synthetic data while mitigating the impact of low-quality trajectories. Analysis shows our model upweights trajectories that align more closely with the real data distribution.

We summarize our contribution as follows:

\begin{enumerate}[leftmargin=*]
    \item We propose a bi-level optimization framework for trajectory-level reweighting of synthetic data in multi-turn RL fine-tuning of LLMs.
    We instantiate this framework with a reweighting head trained jointly with the base LLM and the Q functions via a fully first-order optimization scheme.
    The trajectory-level
reweighting head scales to large datasets without prohibitive computational overhead.
We further introduce some engineering  innovations in bilevel updates which improve training stability.

    
    \item To ground our approach, we derive a PAC-Bayesian bound on performance under trajectory reweighting. This reveals a three-way tradeoff: synthetic data increases diversity but risks task-shift, while focusing on high-quality trajectories improves empirical performance but reduces effective sample size, hiking the complexity penalty.
    

\item We empirically validate our approach on multi-turn RL benchmarks, showing consistent gains over baselines such as uniform data usage and reweighting applied only to real data. While we anticipated reweighting to always bias toward offline data, analysis reveals a surprise: BOOST selectively upweights synthetic trajectories mirroring the true distribution while suppressing low-quality offline samples. This manages the tradeoff between diversity and task-shift without sacrificing effective sample size. Furthermore, this process sometimes also recovers latent signals from offline points that were undervalued prior to synthetic augmentation, suggesting a synergistic rather than competitive relationship between data sources.



\end{enumerate}

\paragraph{Related Work}

Related work spans three areas: synthetic data augmentation and its degradation risks \citep{villalobos2022will, alemohammad2023self, shumailov2023curse}; bi-level optimization for SFT data reweighting \citep{pan2025scalebio, kwon2023fully, chen2025near}; and offline/multi-turn RL for long-horizon LLM planning \citep{snell2022offline, chebotar2023q, hong2025planning}. BOOST uniquely intersects these by extending bi-level reweighting to multi-turn RL with synthetic augmentation. 

\section{Preliminaries: Multi-Turn LLMs Fine-Tuning and Data Selection}
\label{sec:prelims}

\subsection{Task-conditioned Markov Decision Processes} \label{sec:prelims-mdp}

We model multi-turn LLM dialogue as a \emph{task-conditioned} Markov Decision
Process $\mathcal{M}_z = (\mathcal{S}, \mathcal{A}, P_z, R_z, \rho_z, \gamma)$,
where each task $z \sim \mathcal{P}_{\mathrm{task}}$ is drawn from a task
distribution and indexes a task-specific transition kernel $P_z$, reward
$R_z$, and initial-state distribution $\rho_z$
\citep{hallak2015contextual,finn2017model}. The state $s_t$ is the full
interaction history up to time $t$; an action $a_t$ is a complete utterance
sampled from the policy $\pi(\cdot \mid s_t)$; and the next state concatenates
the current state, action, and environment response:
$s_{t+1} = [s_t; a_t; o_t]$. In the \emph{infinite-horizon} formulation,
the goal is to find a policy $\pi_\theta$ maximizing the expected
task-conditioned discounted return
$J_\infty(\pi) = \mathbb{E}_{z \sim \mathcal{P}_{\mathrm{task}},\,\tau \sim p_\pi}
\!\left[\sum_{t=0}^{\infty} \gamma^t R_z(s_t, a_t)\right]$
with discount factor $\gamma \in (0, 1)$.  In the \emph{finite-horizon} formulation, episodes terminate after at most
$H_{\max}$ steps.
Many RL algorithms are traditionally framed for the infinite horizon case, while being often applied in practice to both frameworks.

The common
setup we consider for the theory and experiments in this paper is the
task-conditioned finite-horizon setup with horizon $H_{\max}$, reward
$R_z(s_t, a_t) = -1$ at every step prior to task completion,
$R_z(s_t, a_t) = 0$ once the task is completed (an absorbing terminal
state), and $\gamma = 1$; defining the expected hitting time of policy
$\pi$ on task $z$ as
$T_z^\pi = \mathbb{E}\!\left[\inf\{t \in \mathbb{N} : \text{task $z$ completed at step } t\}\right]$,
with $0 \le T_z^\pi \le H_{\max}$, the cumulative return of any trajectory
under this specification equals $-T_z^\pi$, so maximizing the expected
return $\mathbb{E}_{z\sim\mathcal{P}_{\mathrm{task}}}[J_H(\pi)]$ is
equivalent to minimizing the expected hitting time
$\mathbb{E}_{z\sim\mathcal{P}_{\mathrm{task}}}[T_z^\pi]$.
For notational
simplicity, we omit the explicit task index $z$ for the remainder of this
section when discussing the offline RL objectives and value functions.

\subsection{Offline RL Fine-Tuning for LLMs} \label{sec:prelims-algos}

In the offline setting, we do not interact with the environment but instead learn from a static dataset $D = \{(s_t, a_t, s_{t+1}, r_t)\}$ collected by a behavior policy $\pi_B$. To recover a policy that maximizes long-horizon return without online interaction, we estimate Q-values and V-values from $D$, which allow the policy to be extracted greedily at inference time. We consider two popular approaches:

\textbf{Monte Carlo Return (MC)} \citep{abdulhai2025lmrl} regresses a Q-head directly onto discounted rewards-to-go $G_t = \sum_{k \geq t} \gamma^{k-t} r_k$:
\[
\mathcal{L}_{\text{MC}} = \mathbb{E}_{(s_t, a_t, G_t) \sim D}\left[\bigl(Q_\theta(s_t, a_t) - G_t\bigr)^2\right].
\]

\textbf{Implicit Language Q-Learning (ILQL)} \citep{snell2022offline} jointly learns $Q_\theta(s,a)$ and $V_\theta(s)$ by minimizing Bellman error alongside an expectile loss $L_2^{\tau}$ that enforces $V$--$Q$ consistency:
\[
\mathcal{L}_{\text{ILQL}} = \mathbb{E}_{(s,a,r,s') \sim D}\left[\bigl(r + \gamma V_\theta(s') - Q_\theta(s,a)\bigr)^2\right] + \mathbb{E}_{(s,a) \sim D}\left[L_2^{\tau}\bigl(Q_\theta(s,a) - V_\theta(s)\bigr)\right],
\]

After learning the parameter $\theta$, the final policy is extracted differently depending on the RL objective: under MC, actions are upweighted according to their Q-values, $\pi_\theta(a \mid s) \propto \exp\!\bigl(\beta\, Q_\theta(s,a)\bigr)$, while under ILQL, the policy is extracted using the advantage $A_\theta = Q_\theta(s, a) - V_\theta(s)$, $\pi_\theta(a \mid s) \propto \exp\!\bigl(\beta\, A_\theta(s,a)\bigr)$, where $\beta > 0$ is a temperature parameter in both cases.

\subsection{Data Selection via Bilevel Optimization}
Fine-tuning LLMs is often challenging due to the data quality issue.
Poor data quality can significantly degrade the fine-tuning performance.
Prior work~\citep{pan2025scalebio} proposes using $\boldsymbol{w} \in \Delta^{m}$ to weight $m$ different training sources in $\mathcal{D}_{\text{train}}$ to optimize the validation loss of a validation set $\mathcal{D}_{\text{val}}$:
\begin{equation}~\label{eqn:bilevel-fine-tuning}
\begin{aligned}
    & \min_{\boldsymbol{w}} ~ \mathcal{L}_{\text{RL}}\bigl(\theta^*(\boldsymbol{w}); \mathcal{D}_{\text{val}}\bigr), \quad 
    & \text{s.t.} ~ \theta^*(\boldsymbol{w}) = \arg\min_{\theta} \mathcal{L}_{\text{train}}(\theta, \boldsymbol{w}; \mathcal{D}_\text{train}) \coloneqq \sum\nolimits_{i} w_i \mathcal{L}_\text{RL}(\theta, \mathcal{D}_{\text{train}}^{(i)}),
\end{aligned}
\end{equation}
where the loss function $\mathcal{L}_{\text{RL}}$ can be chosen to be either MC or ILQL fine-tuning loss. Motivated by the prior work in bilevel optimization~\citep{kwon2023fully,chen2025near}, Equation~\ref{eqn:bilevel-fine-tuning} can be reformulated as a minimax problem and trained using iterative gradient descent.
\begin{equation}
    \min_{\theta, \boldsymbol{w}} \max_{\psi} ~
    \mathcal{L}_{\text{RL}}\bigl(\theta^*(\boldsymbol{w}); \mathcal{D}_{\text{val}}\bigr)
    +
    \alpha \bigl(
        \mathcal{L}_{\text{train}}(\theta, \boldsymbol{w}; \mathcal{D}_\text{train})
        -
        \mathcal{L}_{\text{train}}(\psi, \boldsymbol{w}; \mathcal{D}_\text{train})
    \bigr),
    \label{eq:boost-minimax}
\end{equation}
where $\alpha > 0$ is a penalty parameter that enforces the inner-level optimality condition $L_{\text{train}}(\theta, \boldsymbol{w}) \approx L_{\text{train}}(\psi, \boldsymbol{w})$.
In practice, researchers have so far limited weight optimization to a few data sources to avoid the instability inherent in LLM fine-tuning. We address this limitation by showing how to transition from these broad sources to scaling directly to granular data samples.

\section{BOOST: Bilevel Optimization Of Synthetic Trajectories}

\begin{figure}
    \centering
    \includegraphics[width=\linewidth]{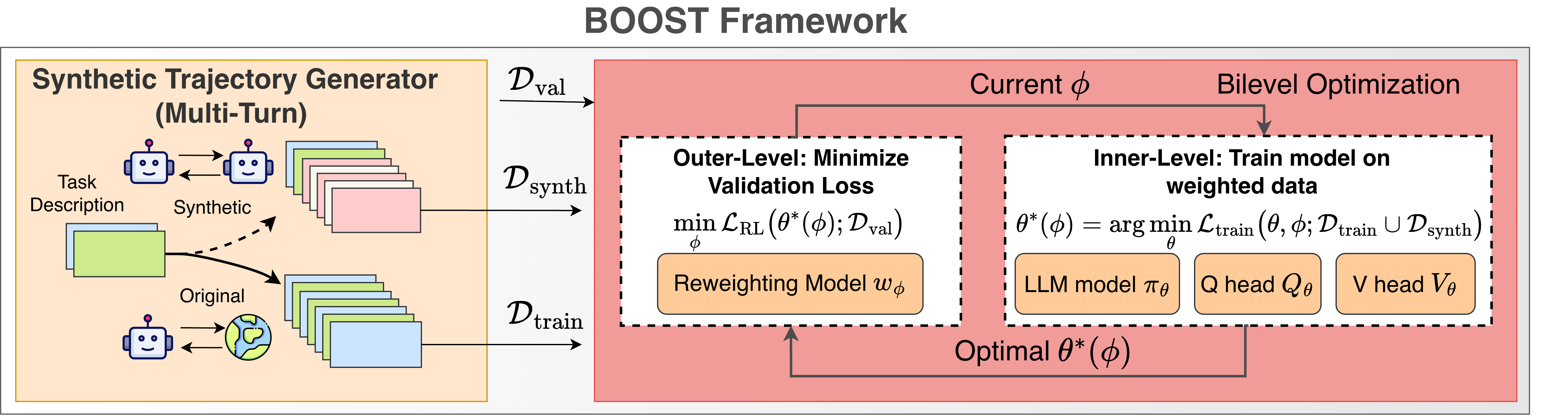}
    \caption{BOOST: The inner level trains model parameters ($\theta$) on a weighted mixture of real and synthetic data, while the outer level updates $\phi$ to minimize loss on a held-out validation set.}
    
    \label{fig:BOOST_framework}
\end{figure}

To overcome the scarcity of high-quality multi-turn dialogue data, we propose \textbf{BOOST} (\textbf{B}ilevel \textbf{O}ptimization of \textbf{S}ynthetic \textbf{T}rajectories) (Figure \ref{fig:BOOST_framework}). The key idea is to generate synthetic data from the multi-turn conversations in the training set, and learn a parameter-efficient \textit{reweight} model to identify useful synthetic trajectories and give them different training weights to improve downstream RL fine-tuning performance (e.g., under MC, or ILQL losses) on real validation tasks.

\subsection{The Challenge of Heterogeneous Synthetic Data}

In modern LLM fine-tuning, to ensure sufficient coverage of the vast state-action space, the dataset $\mathcal{D}$ is often augmented such that $\mathcal{D}_{mixed} = \mathcal{D}_{\text{real}} \cup \mathcal{D}_{\text{synthetic}}$. Because synthetic trajectories are generated by simulators or other LLMs, they exhibit extreme heterogeneity in quality. By optimizing the standard offline RL loss over $\mathcal{D}$ uniformly, the policy $\pi_\theta$ is forced to fit low-quality, biased, or ungrounded synthetic trajectories just as strongly as high-quality real data, ultimately degrading the expected return $J(\pi_\theta)$. 

\subsection{Bilevel Optimization with a Data Reweighting Head}

\begin{algorithm}[ht]
\caption{BOOST: Data-Reweighted Bilevel Optimization}
\label{alg:BOOST_optimization}
\begin{algorithmic}[1]
\REQUIRE Training dataset $\mathcal{D}_{\text{train}}$, validation dataset $\mathcal{D}_{\text{val}}$, synthetic data $\mathcal{D}_{\text{synth}}$
\REQUIRE Outer iterations $N$, inner steps $K_\psi, K_\theta, K_\phi$, Learning rates $\eta_\psi, \eta_\theta, \eta_\phi$; penalty $\alpha$
\STATE Initialize primary policy parameters $\theta$, auxiliary policy parameters $\psi$, reweighting parameters $\phi$
\STATE Construct mixed training set $\mathcal{D}_{\text{mixed}} \gets \mathcal{D}_{\text{train}} \cup \mathcal{D}_{\text{synth}}$
\FOR{$n = 1$ \TO $N$}
    
    \STATE \textit{\% Phase 1: Update auxiliary policy $\psi$ on mixed training data}
    \FOR{$j = 1$ \TO $K_\psi$}
        \STATE $\psi \gets \psi - \eta_\psi \,\nabla_{\psi} \mathcal{L}_{\text{train}}(\psi, \phi; \mathcal{D}_{\text{mixed}})$
    \ENDFOR
    
    \STATE \textit{\% Phase 2: Synchronize primary policy with auxiliary copy}
    \STATE $\theta \gets \psi$
    
    \STATE \textit{\% Phase 3: Update primary policy $\theta$ on validation + mixed training data}
    \FOR{$j = 1$ \TO $K_\theta$}
        \STATE $\mathcal{L}_\theta \gets L_{\text{val}}(\theta; \mathcal{D}_{\text{val}}) \;+\; \alpha\, \mathcal{L}_{\text{train}}(\theta, \phi; \mathcal{D}_{\text{mixed}})$
        \STATE $\theta \gets \theta - \eta_\theta \,\nabla_{\theta} \mathcal{L}_\theta$
    \ENDFOR
    \STATE \textit{\% Phase 4: Update reweighting model $\phi$ via joint minimax loss}
    \FOR{$j = 1$ \TO $K_\phi$}
        \STATE $\mathcal{L}_\phi \gets \mathcal{L}_{\text{val}}(\theta; \mathcal{D}_{\text{val}}) \;+\; \alpha \big( \mathcal{L}_{\text{train}}(\theta, \phi; \mathcal{D}_{\text{mixed}}) - \mathcal{L}_{\text{train}}(\psi, \phi; \mathcal{D}_{\text{mixed}}) \big)$
        \STATE $\phi \gets \phi - \eta_\phi \,\nabla_{\phi} \mathcal{L}_\phi$
    \ENDFOR
\ENDFOR
\end{algorithmic}
\end{algorithm}

Traditional data mixture methods optimize a sampling vector $\boldsymbol{w} \in \Delta^{m}$ over $m$ coarse-grained data sources (e.g., different datasets or task categories). In contrast, BOOST learns a parameter-efficient \emph{data reweighting head} $w_\phi$ that assigns weight to each individual trajectory allowing much finer control over which examples influence learning.

Given $\mathcal D_{mixed}$
and a real validation set $\mathcal{D}_{\text{val}}$,
we train a language model $\pi_\theta$ using a reinforcement-learning-style loss $\mathcal{L}_{\text{RL}}$ (e.g., MC, or ILQL as defined in Section~\ref{sec:prelims-algos}).
The reweighting model $w_\phi$ takes as input the a trajectory $\tau$ and outputs a weight $w_\phi(\tau)$ that modulates its contribution to the training loss.
The weighted training loss is
\[
\mathcal{L}_{\text{train}}(\theta, \phi; \mathcal{D})
\coloneqq
\mathbb{E}_{\tau \in \mathcal{D}}
\left[
\tilde{w}_\phi(\tau) \cdot \mathcal{L}_{\text{RL}}(\theta; \tau)
\right],
\]
where 
$\tilde{w}_\phi(\tau)$ denotes normalized weights computed by applying a softmax over the minibatch $\mathcal{B}$: $\tilde{w}_\phi(\tau) = \exp(w_\phi(\tau)) / \sum_{\tau' \in \mathcal{B}} \exp(w_\phi(\tau'))$).
BOOST thus learns: (1) The language model parameters $\theta$, by minimizing the reweighted loss on real + synthetic training data. (2) The reweighting parameters $\phi$, by encouraging the resulting model $\pi_{\theta^*(\phi)}$ to perform well on real validation data.

\subsection{Efficient and Stable Parameter Updates}
\label{sec:stable-updates}
Algorithm~\ref{alg:BOOST_optimization} presents the full BOOST training procedure. We introduce three key design choices to ensure stable updates and reliable convergence.

\textbf{(1) Alternating Frozen Updates.}
At each inner iteration, we update only one parameter group at a time, the auxiliary policy $\psi$, the primary policy $\theta$, or the reweighting head $\phi$, while keeping the others frozen. This decoupling prevents conflicting gradient signals from destabilizing training. 
Crucially, $\phi$ never receives gradients from a simultaneously shifting policy, and policy updates always see a fixed data weighting. 

\textbf{(2) Warm-Start Initialization of $\theta$ from $\psi$.}
Unlike \citet{kwon2023fully, pan2025scalebio}, where $\theta$ and $\psi$ are updated independently, we re-initialize $\theta \leftarrow \psi$ at the start of each outer iteration. Since $\psi$ tracks the inner-level optimum of $\mathcal{L}_{\text{train}}(\cdot, \phi)$, this gives $\theta$ a warm start that is already near the inner optimum, accelerating convergence and improving stability. This modification introduces no additional computational cost and requires no change to any gradient update rules.

\textbf{(3) Asymmetric Learning Rates Across Levels.}
The policy parameters ($\theta$, $\psi$) use a significantly higher learning rate than the reweighting head $\phi$, with $\phi$ updated up to an order of magnitude more slowly. This reflects the hierarchical structure of bilevel optimization: the lower-level policy must reach a stable, near-optimal solution before the upper-level weights are adjusted. Updating $\phi$ too quickly before the policy has settled introduces spurious gradients that drive the reweighter toward degenerate solutions.

\section{Theoretical Analysis: Trajectory Reweighting and Task Generalization}
\label{sec:theory}

As established in Section~\ref{sec:prelims}, we adopt the finite-horizon,
task-conditioned hitting-time specification: tasks $z$ are drawn from
$\mathcal{P}_{\mathrm{task}}$, each inducing a task-specific MDP
$\mathcal{M}_z$, and $T_z^\pi$ denotes the expected hitting time of policy
$\pi$ on task $z$, with $0 \le T_z^\pi \le H_{\max}$. Recall that, under this
specification, maximizing expected return is equivalent to minimizing
$\mathbb{E}_{z \sim \mathcal{P}_{\mathrm{task}}}[T_z^\pi]$, so the analysis
below bounds the latter.


PAC--Bayesian bounds compare a learned predictor to a fixed prior chosen before
training. In our setting, the natural prior is the current pretrained LLM before
task-specific RL fine-tuning. We denote by $P_0$ the corresponding prior over
policies, centered at the pretrained LLM $\pi_0$. 


Let $\mathcal{Q}_\theta$ denote the posterior over stochastic policies, induced by the
fine-tuned LLM with weights $\theta$. To keep notation compact, write
$T_z^{\pi_\theta}
\coloneqq
\mathbb E_{\pi\sim \mathcal{Q}_\theta}[T_z^\pi]$.
Let $\mathcal P_\phi$ denote the weighted mixed task distribution induced by the
fixed trajectory weights $\tilde w_\phi$ over real and synthetic data. Define
\[
\Delta_{\mathrm{task}}(\phi)
=
\sup_{\pi\in\Pi}
\left|
\mathbb E_{z\sim\mathcal P_{\mathrm{task}}}[T_z^\pi]
-
\mathbb E_{z\sim\mathcal P_\phi}[T_z^\pi]
\right|.
\]
This term measures how much the weighted mixed distribution differs from the
target task distribution for the policy class under consideration.
 

We define effective sample size as $n_{\mathrm{eff}}(\tilde{w}_\phi) = \frac{1}{(\sum_{i=1}^n \tilde{w}_{\phi,i}^2)}$, where $\tilde w_{\phi,i}$ are trajectory weights.

\begin{theorem}[Weight-conditional PAC--Bayesian hitting-time bound]
\label{thm:weighted_pac_hitting_time}
Fix normalized weights $\tilde w_\phi$ over $\mathcal D_{\mathrm{mixed}}$.
Let $\mathcal P_0$ be the pretrained-policy prior (chosen independently of the
finetuning data), and let $\mathcal Q_\theta$ be the posterior induced
by the learned policy. Then there exists   a universal constant $C_1>0$ such that, for any $\delta\in(0,1)$, with probability at least
$1-\delta$,
\begin{equation}
\begin{aligned}
\mathbb E_{z\sim\mathcal P_{\mathrm{task}}}
\!\left[T_z^{\pi_\theta}\right]
\le\;&
\underbrace{
\sum_{i=1}^{n}
\tilde w_{\phi,i}T_{z_i}^{\pi_\theta}
}_{\text{weighted empirical time}}
+
\underbrace{
C_1H_{\max}
\sqrt{
\frac{
\mathrm{KL}(\mathcal Q_\theta\|\mathcal P_0)+\log(2/\delta)
}{
n_{\mathrm{eff}}(\tilde w_\phi)
}
}
}_{\text{policy complexity}}
+
\underbrace{
\Delta_{\mathrm{task}}(\phi)
}_{\text{task-shift penalty}} .
\end{aligned}
\label{eq:weighted_pac_hitting_time}
\end{equation}
\end{theorem}

\begin{proof}[Proof sketch]
For fixed weights, apply the PAC--Bayesian Hoeffding--Azuma inequality of
\citet{seldin2012pac} to the bounded loss
$T_z^\pi/H_{\max}\in[0,1]$. The weighted centered losses form martingale
differences, and the $i$th increment has range length
$\tilde w_{\phi,i}$. Thus the Hoeffding--Azuma range term is
$\sum_i\tilde w_{\phi,i}^2=1/n_{\mathrm{eff}}(\tilde w_\phi)$.
Multiplying by $H_{\max}$ gives a bound on
$\mathbb E_{z\sim\mathcal P_\phi}[T_z^{\pi_\theta}]$ in terms of the weighted
empirical hitting time and the PAC--Bayesian complexity term. Finally,
\[
\mathbb E_{z\sim\mathcal P_{\mathrm{task}}}[T_z^{\pi_\theta}]
\le
\mathbb E_{z\sim\mathcal P_\phi}[T_z^{\pi_\theta}]
+
\Delta_{\mathrm{task}}(\phi).
\]
This gives Eq.~\eqref{eq:weighted_pac_hitting_time}. The proof is in
Appendix~\ref{app:weighted_pac_bayes}, which shows a slightly sharper statement
by applying the PAC--Bayesian Hoeffding--Azuma inequality of
\citet{seldin2012pac}.
\end{proof}

The bound is stated for a fixed weight vector $\tilde w_\phi$ and analyzes the inner learning problem once the outer BOOST reweighting model has produced weights. Adding synthetic data
uniformly can increase coverage and effective sample size, but only when those
trajectories have low hitting time and do not increase the task-shift penalty.
Conversely, placing nearly all mass on a small number of apparently good
trajectories can reduce the empirical term but lowers
$n_{\mathrm{eff}}$, increasing the complexity penalty. BOOST aims for the useful
middle regime: it assigns weight to synthetic trajectories that improve
validation performance, discounts trajectories that introduce poor behavior or
task shift, and preserves enough effective sample size for the resulting policy
to generalize to new tasks.

\section{Experiments}



We evaluate BOOST on three multi-turn dialogue environments: two from the \textbf{LMRL-Gym} benchmark \citep{abdulhai2025lmrl}, 
testing strategic reasoning and sequential decision-making through natural language,
plus a benchmark for multi-turn LLM tool use in a realistic web navigation setting. 
Although each environment specifies its reward in its own natural units
(questions-remaining, turns-remaining, or step-discounted match score),
in all three cases the reward is equivalent to the finite-horizon
hitting-time setup of Section~\ref{sec:prelims-mdp}.


\paragraph{Guess (20 Questions).}
A strategic information-gathering task from LMRL-Gym with 40K offline dialogues where an agent tries to identify a hidden object by posing yes/no questions to an oracle. Each question receives a binary response; the agent must iteratively narrow the hypothesis space to find the object. Effective policies ask maximally informative questions to efficiently eliminate candidate hypotheses. \textit{Reward:} The agent receives points equal to questions remaining upon correct identification; incorrect guesses or exhausting the limit yield zero reward. 


\paragraph{Car Dealer Negotiation.}
An LMRL-Gym bargaining task with 4.4K offline dialogues where the agent acts as a car salesperson negotiating with a buyer.
The agent proposes prices and responds to buyer counteroffers across multiple dialogue turns until 
an agreement or rejection occurs. The task requires strategic persuasion and principled price management.
\textit{Reward:} Rewards depend on the final negotiation outcome and how quickly it is reached:
reward equals the number of dialogue turns remaining at the point of agreement; failed negotiations yield zero reward.

\paragraph{WebShop Agent.}
WebShop \citep{yao2022webshop} is an online shopping benchmark featuring diverse structured and unstructured text. This including product titles, descriptions, and attributes—and requires the agent to complete purchase tasks based on natural language queries (e.g., \textit{``I am looking for a nightstand with drawers, a nickel finish, and a price below \$140''}). The agent has multi-turn interaction with a simulated webpage through search queries, clicks, filter selections, and page navigation. We adapted the original online benchmark
for our offline setting, we use GPT-4 to collect 2,000 realistic interaction trajectories as training data.
\textit{Reward:} Performance is measured using two metrics: attribute match score (the proportion of desired product attributes satisfied, averaged across episodes) and success rate (the percentage of episodes in which all requirements are fully met). The reward signal used during training is the attribute match score, discounted by steps to checkout.


\paragraph{Data Generation}

To generate synthetic data $\mathcal{D}_{synth}$ for multi-turn RL, we use offline trajectories as a starting point. 
For each environment, we initialize from an offline trajectory's first two steps, then simulate the remainder by prompting an LLM to alternate roles as agent and environment. The trajectory ends at the turn limit or natural conclusion. No external simulator is used.
To ensure a fair comparison, we use the same LLM that is subsequently fine-tuned for data generation.

A central goal of this work is to evaluate generalization from limited training data to unseen task categories. Thus, we partition the data by task category rather than by individual task instance. In the 20 Questions task, categories correspond to the semantic class of the target word (e.g., tools, fruits); in Car Dealer Negotiation, they correspond to car brands, which implicitly encode information such as price range and feature set; in WebShop, they correspond to product classes (e.g., beauty products, apparel). Offline trajectories and the synthetic data derived from them are split by category into train, validation, and evaluation sets. The validation set shares the same category distribution as the evaluation set but is kept deliberately small, reflecting realistic conditions in which only limited labeled data from target tasks is available. This setup directly tests the model's ability to generalize across task categories rather than merely interpolate within them. Further details in the Appendix.

\subsection{Evaluation}

For each environment, we train using two RL algorithms: Monte Carlo Regression and ILQL. Within each algorithm, we evaluate four methods:

\begin{enumerate}[leftmargin=*]
    \item \textbf{Vanilla.} Training uses only the available offline dataset.
    \item \textbf{Vanilla+Synthetic.} Training uses the offline dataset augmented with synthetic trajectories.
    \item \textbf{Bilevel Only.} Training uses only offline data, but optimizes sample weights via bilevel updates on the train/validation split.
    \item \textbf{BOOST.} Training uses offline data augmented with synthetic trajectories, optimized via the bilevel procedure described in Algorithm~\ref{alg:BOOST_optimization}.
\end{enumerate}

To assess the effect of data availability, each regime is evaluated under both a \textbf{high-data} and a \textbf{low-data} setting, where the low-data setting uses an order of magnitude fewer training examples than the high-data setting. This directly probes each method's robustness under the limited training data conditions that motivate this work.

While all methods train on offline data, policies are evaluated by interacting with a live simulated environment over a set of held-out goals. This constitutes a deliberately challenging evaluation regime that tests generalisation across unseen task categories and exposes each policy to the distribution shift inherent in moving from static, logged trajectories to dynamic online interaction where the policy's own outputs shape subsequent environment responses in ways not observed during training, more closely reflecting the conditions these policies would face in real-world deployment.

For the 20 Questions and Car Dealer Negotiation tasks, we fine-tune a GPT-2 model, which has been shown to perform adequately on these tasks in prior work~\cite{abdulhai2025lmrl}. WebShop is a substantially more challenging environment, requiring multi-step information gathering and diverse action types. Thus, we fine-tune a Llama-3-8B model and a Qwen-3.5-9B model using parameter-efficient fine-tuning via QLoRA. To ensure fair comparison across methods, the LLM used to generate synthetic data is always identical to the one being fine-tuned. 
Further implementation details are in the Appendix.









\section{Results and Discussion}

\begin{table}[t]
\centering
\caption{
    Evaluation reward of MC Return and ILQL across environments, LLM models, and data regimes
    (mean $\pm$ standard error over 64 evaluation rollouts; higher is better).
    Bilevel reweighting outperforms vanilla in nearly all settings, and
    BOOST is the best method in all low-data settings.
}
\vspace{2pt}
\label{tab:main_results_full}
\setlength{\tabcolsep}{3.5pt}
\renewcommand{\arraystretch}{1.1}
\small
\begin{tabular}{lll cccc}
\toprule
\textbf{Data} & \textbf{Algo} & \textbf{Environment (LLM)}
    & \textbf{Vanilla}
    & \textbf{Vanilla+Syn}
    & \textbf{Bilevel Only}
    & \textbf{BOOST (ours)} \\
\midrule
\multirow{8}{*}{Low}
  & \multirow{4}{*}{MC}
    & WebShop (Qwen 3.5-9B)     & $0.23 \pm 0.06$ & $0.28 \pm 0.04$ & $0.41 \pm 0.04$ & $\mathbf{0.43 \pm 0.03}$ \\
  & & WebShop (Llama-3-8B)    & $0.24 \pm 0.02$ & $0.23 \pm 0.02$ & $0.32 \pm 0.03$ & $\mathbf{0.40 \pm 0.03}$ \\
  & & Car Negotiation (GPT-2) & $26.73 \pm 0.39$ & $31.08 \pm 0.00$ & $36.28 \pm 2.23$ & $\mathbf{40.92 \pm 1.86}$ \\
  & & Twenty Questions (GPT-2) & $1.01 \pm 0.13$ & $1.01 \pm 0.11$ & $2.90 \pm 0.22$ & $\mathbf{3.46 \pm 0.22}$ \\
\cmidrule(lr){2-7}
  & \multirow{4}{*}{ILQL}
    & WebShop (Qwen 3.5-9B)     & $0.30 \pm 0.05$ & $0.33 \pm 0.04$ & $0.33 \pm 0.04$ & $\mathbf{0.40 \pm 0.04}$ \\
  & & WebShop (Llama-3-8B)    & $0.30 \pm 0.02$ & $0.27 \pm 0.02$ & $0.32 \pm 0.03$ & $\mathbf{0.38 \pm 0.02}$ \\
  & & Car Negotiation (GPT-2) & $27.22 \pm 0.19$ & $25.94 \pm 0.93$ & $29.08 \pm 2.01$ & $\mathbf{34.91 \pm 2.37}$ \\
  & & Twenty Questions (GPT-2) & $0.38 \pm 0.08$ & $0.45 \pm 0.09$ & $0.51 \pm 0.11$ & $\mathbf{0.53 \pm 0.10}$ \\
\midrule
\multirow{6}{*}{High}
  & \multirow{3}{*}{MC}
    & WebShop (Llama-3-8B)    & $0.38 \pm 0.02$ & $0.31 \pm 0.02$ & $0.43 \pm 0.03$ & $\mathbf{0.45 \pm 0.03}$ \\
  & & Car Negotiation (GPT-2) & $32.27 \pm 0.39$ & $35.09 \pm 0.39$ & $41.31 \pm 1.71$ & $\mathbf{41.42 \pm 1.71}$ \\
  & & Twenty Questions (GPT-2) & $2.22 \pm 0.18$ & $2.19 \pm 0.23$ & $\mathbf{3.83 \pm 0.23}$ & $3.48 \pm 0.21$ \\
\cmidrule(lr){2-7}
  & \multirow{3}{*}{ILQL}
    & WebShop (Llama-3-8B)    & $0.31 \pm 0.02$ & $0.27 \pm 0.02$ & $\mathbf{0.41 \pm 0.02}$ & $0.36 \pm 0.02$ \\
  & & Car Negotiation (GPT-2) & $35.33 \pm 0.78$ & $33.00 \pm 0.78$ & $\mathbf{37.55 \pm 2.40}$ & $37.36 \pm 2.39$ \\
  & & Twenty Questions (GPT-2) & $2.60 \pm 0.21$ & $2.04 \pm 0.19$ & $5.38 \pm 0.38$ & $\mathbf{5.83 \pm 0.28}$ \\
\bottomrule
\end{tabular}
\end{table}

In Table \ref{tab:main_results_full}, we report test-time reward performance for the MC Returns and ILQL fine-tuning methods averaged over 64 evaluation rollouts. We discuss several key findings below.

\paragraph{Bilevel reweighting consistently improves performance.} Across both fine-tuning methods, all three environments, and both data regimes, BOOST and Bilevel strictly outperform both vanilla versions. This consistent pattern strongly suggests that offline datasets are heterogeneous in quality; reweighting examples to optimize downstream validation performance is a reliable and general improvement over uniform training. The margins are often substantial: in the MC Returns setting , BOOST improves over both vanilla settings by a substantial margin in the low-data regime, successfully completing the 20 Questions task an average of $\sim2$ turns earlier and the Car Dealer Negotiation an average of $\sim9$ turns earlier. This improvement is also statistically significant.

\paragraph{Naively adding synthetic data can hurt}. While more data generally aids generalization, unweighted synthetic augmentation is unreliable. In the ILQL setting, Vanilla with Synthetic underperforms Vanilla in four out of seven settings, Car Negotiation high-data (
33.00 vs
35.33), low-data (25.94 vs 27.22) and both Webshop regimes (
0.27 vs 0.30/0.31). A similar pattern appears in MC Returns, Vanilla with Synthetic performs slightly poorly in Twenty Questions high-data (2.19 vs 2.22) and Webshop high-data (0.31 vs 0.38). This demonstrates that synthetic data introduces noise that, without proper reweighting, the model cannot filter out.

\paragraph{Comparison with Data-Selection Baselines}
\label{sec:baselines}

\begin{table}[t]
\centering
\caption{Data-selection baselines vs.\ BOOST in the low-data regime (evaluation reward, higher is better).
WebShop uses Llama-3-8B; Car Negotiation and Twenty Questions (20Q) use GPT-2.}
\vspace{2pt}
\label{tab:baselines}
\setlength{\tabcolsep}{3pt}
\small
\begin{tabular}{l cccccc}
\toprule
\textbf{Method} & \textbf{Web MC} & \textbf{Web ILQL} & \textbf{Car MC} & \textbf{Car ILQL} & \textbf{20Q MC} & \textbf{20Q ILQL} \\
\midrule
Vanilla+Syn & 0.23 & 0.27 & 31.08 & 25.94 & 1.01 & 0.45 \\
Binary filter~\citep{singh2023beyond} & 0.25 & 0.26 & 20.08 & 22.58 & 2.18 & 0.50 \\
LLM-agent weighting~\citep{chen2024alpagasus} & 0.23 & 0.27 & 19.94 & 19.00 & 1.06 & \textbf{0.53} \\
ScaleBiO~\citep{pan2025scalebio} & 0.35 & 0.33 & \textbf{41.58} & 29.06 & 2.80 & 0.51 \\
BOOST (ours) & \textbf{0.40} & \textbf{0.38} & 40.92 & \textbf{34.91} & \textbf{3.46} & \textbf{0.53} \\
\bottomrule
\end{tabular}
\end{table}

We compare BOOST against three data-selection baselines in the low-data regime of all environments, for both MC and ILQL.
\textbf{ScaleBiO source weighting}~\citep{pan2025scalebio} learns two scalar weights, one for offline and one for synthetic trajectories, using the same bilevel objective on the same real validation data; it differs from BOOST only in reweighting granularity.
\textbf{Binary filtering}, following \citet{singh2023beyond}, discards synthetic trajectories judged unsuccessful by the same generator that produced them, then trains on the union of offline data and the retained synthetic data.
\textbf{LLM-agent weighting}, following \citet{chen2024alpagasus}, uses a strong external LLM agent (Claude Opus 4.8) to identify and reweight low-quality synthetic trajectories; the resulting weights are used as sampling weights during fine-tuning.

Table~\ref{tab:baselines} shows that BOOST is the best or tied-best method in five of six settings. 
BOOST significantly outperforms Binary filter in all settings and performs significantly better than ScaleBiO and Agent-based reweighting in the majority of settings. Notably, ScaleBiO uses the same validation-driven bilevel objective on the same data, so BOOST’s gain comes specifically from per-trajectory granularity, not merely from the presence of weighting. Both fixed-signal heuristics (binary success-filtering and strong LLM agent quality scoring) perform poorly in most settings. This shows that neither a reward heuristic nor a strong external LLM-agent substitutes for learned reweighting.


\paragraph{BOOST achieves strongest results in low-data regimes.}
By combining bilevel reweighting with synthetic augmentation, BOOST yields the best performance in 11 out of 14 settings. The advantage is most pronounced in the low-data regime, where it is the clear winner across all eight settings, outperforming even Bilevel thus demonstrating that synthetic data provides genuine
additional value when reweighted appropriately.
The biggest gains are under MC Returns where BOOST improves over the best vanilla baseline by 242.6\% in Twenty Questions (3.46 vs 1.01), 31.7\% in Car Negotiation (40.92 vs. 31.08), and 66.7\% in Webshop (0.40 vs. 0.24). This confirms our motivation: when real trajectories are scarce, synthetic data increases conversational diversity, while bilevel reweighting ensures that only synthetic transitions genuinely aligned with the validation objective receive high weight.

\begin{figure*}[] 
    \centering
    \hfill
    \begin{subfigure}[t]{0.30\textwidth}
        \centering
        \includegraphics[width=\textwidth]{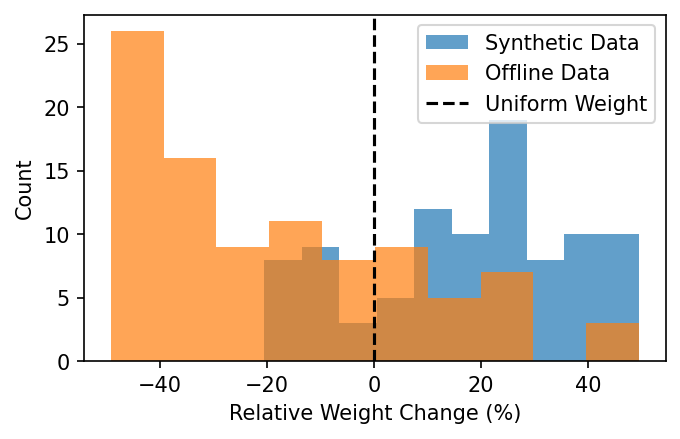}
        \caption{Twenty Questions}
        \label{fig:weight_distra}
    \end{subfigure}
    \begin{subfigure}[t]{0.30\textwidth}
        \centering
        \includegraphics[width=\textwidth]{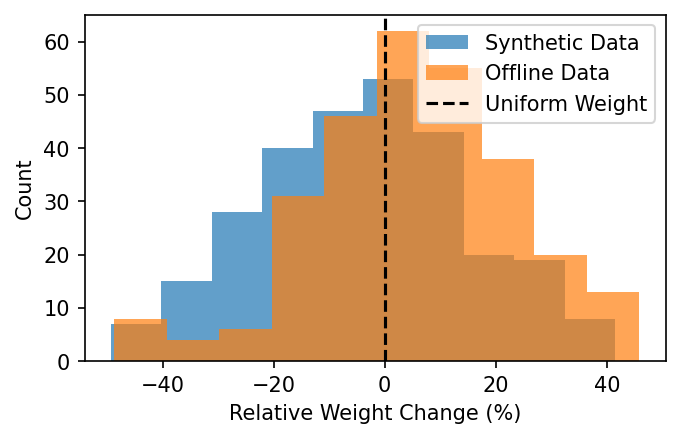}
        \caption{Car price Negotiation.}
        \label{fig:weight_distrb}
    \end{subfigure}
    \begin{subfigure}[t]{0.30\textwidth}
        \centering
        \includegraphics[width=\textwidth]{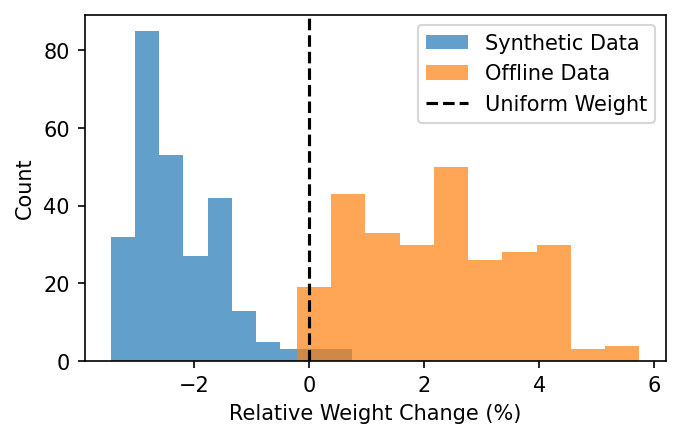}
        \caption{ Webshop Agent}
        \label{fig:weight_distrc}
    \end{subfigure}
    \hfill
    \label{fig:weight_distr}
    \centering
    \hfill
    \begin{subfigure}[t]{0.30\textwidth}
        \centering
        \includegraphics[width=\textwidth]{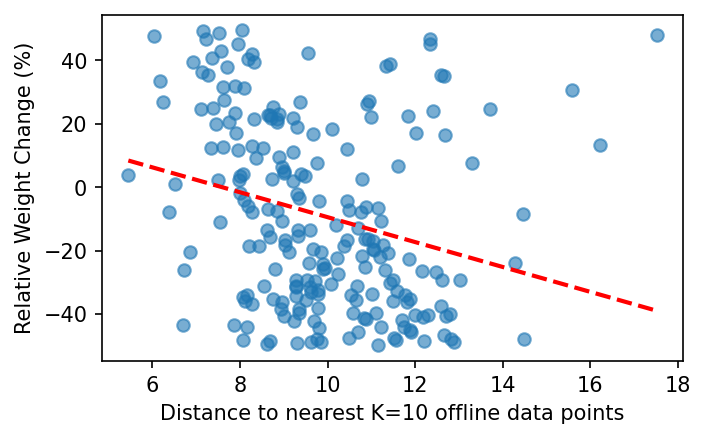}
        \caption{Twenty Questions}
        \label{fig:weight_distancea}
    \end{subfigure}
    \begin{subfigure}[t]{0.30\textwidth}
        \centering
        \includegraphics[width=\textwidth]{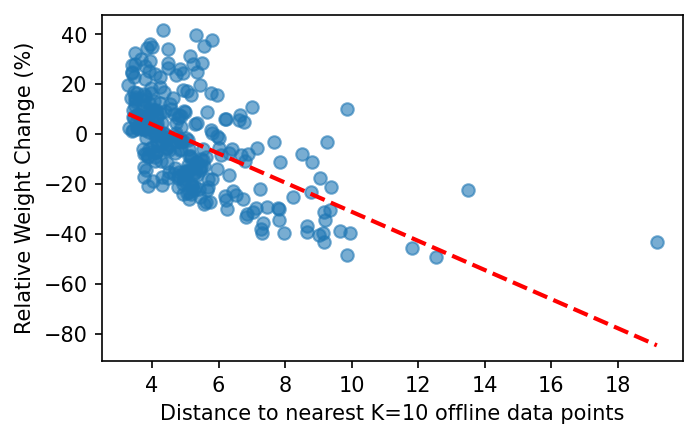}
        \caption{Car Price Negotiation}
        \label{fig:weight_distanceb}
    \end{subfigure}
    \begin{subfigure}[t]{0.30\textwidth}
        \centering
        \includegraphics[width=\textwidth]{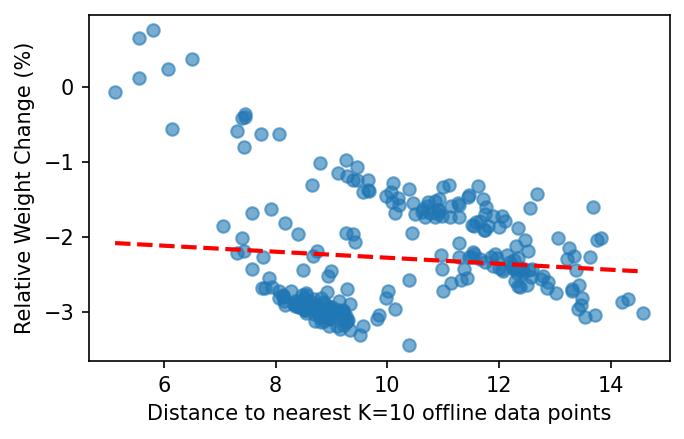}
        \caption{Webshop Agent}
        \label{fig:weight_distancec}
    \end{subfigure}
    \hfill

    \caption{Learned weight distributions across environments. \textit{Top:} relative weight change
from uniform initialization for synthetic and offline points. \textit{Bottom:} weight
change vs.\ distance to nearest offline neighbours.}
    \label{fig:weight_all}
\end{figure*}


\paragraph{Analyzing Learnt Weights} 

\textit{Weights are non-uniform and vary on a per-point basis.}
Initially, all data points are assigned uniform sampling weights. Over the course of BOOST training, the data reweighting head $w_\phi$ learns heterogeneous weights across data points. In Figures~\ref{fig:weight_distra}--\ref{fig:weight_distrc}, we plot the relative change in weights from the uniform initialization, grouped separately for synthetic and offline data. Two observations stand out.
First, the learned weights vary substantially across data points, confirming the bilevel objective identifies differential utility rather than collapsing to a uniform solution. Second, weights do not separate by source: some synthetic points receive high weight while some offline points are downweighted.
This demonstrates that $w_\phi$ cannot be approximated by any simple heuristic that uniformly downweights synthetic data and upweights offline data. 
Instead, the head performs fine-grained, per-point discrimination based on validation signals. This underscores the need for BOOST's learned reweighting over simpler source-based filtering or fixed mixing ratios.

\textit{Highly-weighted synthetic points are those closest to the offline distribution.}
In Figures~\ref{fig:weight_distancea}--\ref{fig:weight_distancec}, we test the hypothesis that synthetic data points receiving high weight are those most similar to true offline data points. For each synthetic data point, we plot the relative weight change from uniform against its mean $L_1$ distance in embedding space to its $K\!=\!10$ nearest offline neighbours. Across all environments, weight and distance from offline data are inversely related: synthetic points closer to the offline distribution receive higher weights, while outliers are discounted. Pearson correlations confirm this trend, showing negative correlations of $-0.26$ (Twenty Questions), $-0.49$ (Car Dealer), and $-0.10$ (Webshop). This evidence confirms BOOST doesn't indiscriminately upweight synthetic data; instead, it selectively incorporates transitions most consistent with real offline data.

\begin{wrapfigure}{r}{0.35\textwidth}
    \centering
    \vspace{-10pt}  
    \includegraphics[width=0.35\textwidth]{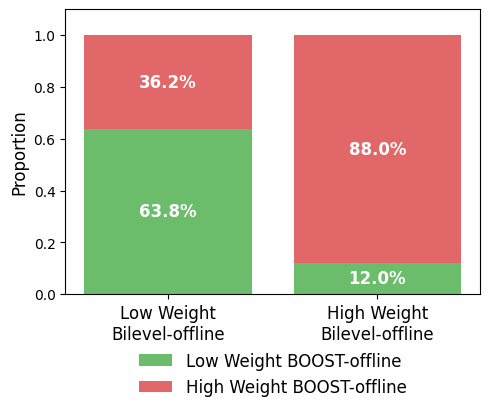}
    \caption{Car Dealer environment weight shift from Bilevel to BOOST}
    \label{fig:weight_stack_car}
    \vspace{-10pt}  
\end{wrapfigure}


\textit{Adding synthetic data rescues informative offline points.}
Finally, we examine how the weights assigned to offline data points shift when synthetic data is introduced into BOOST. Concretely, we compare offline point weights under Bilevel-Offline-Only against their weights under BOOST, and ask whether points that were highly weighted under the offline-only objective remain highly weighted when synthetic data is added. In the Car Dealer environment (see Figure \ref{fig:weight_stack_car}), $88\%$ of highly-weighted offline points retain high weight under BOOST, demonstrating that the bilevel objective is stable and does not discard previously identified high-quality transitions. More strikingly, among offline points that received low weight under the offline-only objective, BOOST promotes $36.2\%$ of them to high weight. This suggests that these points possessed latent signal that was previously obscured, likely due to distributional gaps in the offline dataset, and that the additional coverage provided by synthetic data gives the bilevel optimizer sufficient context to recognise and recover their value.

In addition to these three analyses, we provide qualitative comparison of synthetic trajectories with high and low weights in Appendix~\ref{app:qualitative} . We also show ablations on BOOST's design choices in Appendix~\ref{app:ablation}. Together, they showcase that BOOST learns to upweight offline data that is inherently informative, selectively incorporates synthetic data that resembles the offline distribution, and recovers signal from offline points that were undervalued in the absence of synthetic augmentation.

\section{Conclusion}
We presented BOOST, a bilevel framework for improving multi-turn offline RL fine-tuning
by learning trajectory-level weights over mixed real and synthetic data. Across three
environments and two RL algorithms, BOOST consistently outperforms uniform training and
static synthetic augmentation, with the largest gains in low-data regimes. Analysis shows
the learned weights are semantically meaningful: the reweighting head upweights synthetic
trajectories nearest to the real distribution, recovers signal from undervalued offline
points, and suppresses noisy transitions without explicit quality supervision. Together
with our PAC-Bayesian hitting-time bound, these results establish principled
synthetic-data reweighting as a practical path toward data-efficient multi-turn LLM
adaptation. A current limitation is that reweighting operates at the trajectory level;
finer-grained per-timestep reweighting, which could recover useful signal from partially
good trajectories, is a natural direction for future work.

\begin{ack}
This work was supported by ONR MURI N00014-24-1-2742, by NSF IIS-2403240, NSF IIS-2552007, Schmidt Sciences AI2050 Fellowship and based upon work supported by the AI Research Institutes Program funded by the National Science Foundation
under the AI Institute for Societal Decision Making (NSF AI-SDM), Award No. 2229881.
\end{ack}

\bibliographystyle{plainnat}
\bibliography{references}


\appendix

\section{Background on PAC--Bayesian Bounds}
\label{app:pac_bayes_background}

PAC--Bayesian analysis gives high-probability generalization bounds for
data-dependent randomized predictors \citep{mcallester1999pac,seeger2002pac,catoni2007pac,germain2009pac}. Let $\Pi$ be a policy class, let $P_0$ be a prior distribution over $\Pi$, and let $Q$ be a posterior distribution over $\Pi$ chosen after observing the training sample. The prior $P_0$ must be fixed independently of the sample used in the bound. For a bounded loss $\ell(z,\pi)\in[0,1]$, define
\[
L(Q)
=
\mathbb E_{\pi\sim Q}
\mathbb E_{Z\sim\mathcal P}[\ell(Z,\pi)],
\qquad
\widehat L(Q)
=
\mathbb E_{\pi\sim Q}
\left[
\frac{1}{n}\sum_{i=1}^n \ell(Z_i,\pi)
\right].
\]
In the standard i.i.d. setting, where $Z_1,\ldots,Z_n\sim\mathcal P$, PAC--Bayesian bounds control $L(Q)$ by $\widehat L(Q)$ plus a complexity term of the form
\[
O\!\left(
\sqrt{
\frac{\mathrm{KL}(Q\|P_0)+\log(1/\delta)}{n}
}
\right).
\]
The KL term measures how far the learned posterior moves from the prior, while the $1/\sqrt n$ factor comes from concentration of the uniform empirical average.

In our setting, predictors are stochastic policies. The natural prior is a distribution over policies centered at the pretrained LLM before task-specific RL fine-tuning. This is the reference policy available before the task-specific trajectories are observed. The posterior $Q_\theta$ is the distribution over policies induced by the fine-tuned LLM. Thus $\mathrm{KL}(Q_\theta\|P_0)$ measures deviation from pretrained behavior, rather than deviation from a uniform distribution over utterances.

The technical issue is that BOOST does not use the uniform empirical loss. For fixed normalized weights $w\in\Delta_n$, the empirical loss is
\[
\widehat L_w(Q)
=
\mathbb E_{\pi\sim Q}
\left[
\sum_{i=1}^n w_i\ell(Z_i,\pi)
\right].
\]
The standard i.i.d. PAC--Bayesian bound does not directly apply to this weighted average by simply replacing $n$ with an effective sample size. The concentration term must account for the nonuniform ranges of the weighted summands.

Section~\ref{app:weighted_pac_bayes} handles this point formally. We apply the PAC--Bayesian Hoeffding--Azuma inequality of \citet{seldin2012pac} to the weighted centered losses. In the independent-sample setting considered here, these centered losses form martingale differences, and their range lengths are exactly the weights $w_i$. This yields the weighted PAC--Bayesian bound in Corollary~\ref{cor:weighted_pac_bayes_bounded}, the exact Seldin-form hitting-time bound in Theorem~\ref{thm:exact_seldin_weighted_hitting_time}, and the simplified constant-form bound used in the main text in Corollary~\ref{cor:simplified_weighted_hitting_bound}.

\begin{remark}[On the independence assumption]
\label{rem:independence}
Corollary~\ref{cor:weighted_pac_bayes_bounded} assumes $Z_i \perp (Z_1,\ldots,Z_{i-1})$. In our experimental protocol (Section~\ref{sec:theory} data generation), synthetic trajectories are initialized from prefixes of real trajectories, so some synthetic tasks $z_j^{\mathrm{synth}}$ are not independent of real tasks $z_i^{\mathrm{real}}$. The martingale construction extends to this case by placing all real tasks first in the filtration $\{\mathcal F_i\}$ and treating synthetic tasks as conditionally independent given the real-task $\sigma$-algebra $\mathcal G$; $\mathcal P_\phi$ then becomes a $\mathcal G$-random measure, and the concentration bound holds conditionally on $\mathcal G$ with the same constants. We omit the measurability bookkeeping.
\end{remark}

\section{Proof of the Weighted Hitting-Time Bound}
\label{app:weighted_pac_bayes}

We prove the bound used in Section~\ref{sec:theory}. Throughout this section,
the normalized weights are fixed. Thus the result is weight-conditional: it
analyzes the inner learning problem after the outer BOOST reweighting model has
produced a weight vector.

\paragraph{Auxiliary result.}
We use the following range-length form of the PAC--Bayesian Hoeffding--Azuma
inequality of \citet{seldin2012pac}.

\begin{theorem}[PAC--Bayesian Hoeffding--Azuma; \citealp{seldin2012pac}]
\label{thm:seldin_pac_bayes_azuma}
Let $\Pi$ be a policy class and let $P_0$ be a prior distribution over $\Pi$
chosen independently of the training sample. For each $\pi\in\Pi$, let
$X_1(\pi),\ldots,X_n(\pi)$ be martingale differences with respect to a
filtration $\{\mathcal F_i\}_{i=0}^n$:
\[
\mathbb E[X_i(\pi)\mid \mathcal F_{i-1}]=0.
\]
Assume that, for each $i$, there is a deterministic constant $r_i\ge0$ such
that for every $\pi\in\Pi$, the random variable $X_i(\pi)$ is almost surely
contained in an interval of length at most $r_i$. Then, for any $1<c\le e$,
with probability at least $1-\delta$, every posterior distribution $Q$ over
$\Pi$ satisfies
\[
\left|
\mathbb E_{\pi\sim Q}
\left[
\sum_{i=1}^n X_i(\pi)
\right]
\right|
\le
\frac{1+c}{2\sqrt{2}}
\sqrt{
\left(
\mathrm{KL}(Q\|P_0)
+
\log\frac{2}{\delta}
+
\varepsilon_c(Q)
\right)
\sum_{i=1}^n r_i^2
},
\]
where
\[
\varepsilon_c(Q)
=
\frac{\log 2}{2\log c}
\left[
1+
\log\left(
1+
\frac{\mathrm{KL}(Q\|P_0)}{\log(2/\delta)}
\right)
\right].
\]
\end{theorem}

\begin{corollary}[Weighted PAC--Bayesian bound for bounded losses]
\label{cor:weighted_pac_bayes_bounded}
Let $\ell(z,\pi)\in[0,1]$ be a bounded loss. Let
$Z_1,\ldots,Z_n$ be independent tasks with $Z_i\sim\mathcal P_i$.
Fix weights $w\in\Delta_n$, where $w_i\ge0$ and $\sum_{i=1}^n w_i=1$.
Define
\[
\widehat L_w(\pi)
=
\sum_{i=1}^n w_i\ell(Z_i,\pi),
\qquad
L_w(\pi)
=
\sum_{i=1}^n
w_i
\mathbb E_{Z\sim\mathcal P_i}[\ell(Z,\pi)],
\]
and
\[
n_{\mathrm{eff}}(w)
=
\frac{1}{\sum_{i=1}^n w_i^2}.
\]
Then, for any $1<c\le e$, with probability at least $1-\delta$, every posterior
$Q$ over $\Pi$ satisfies
\[
\mathbb E_{\pi\sim Q}[L_w(\pi)]
\le
\mathbb E_{\pi\sim Q}[\widehat L_w(\pi)]
+
\frac{1+c}{2\sqrt{2}}
\sqrt{
\frac{
\mathrm{KL}(Q\|P_0)
+
\log\frac{2}{\delta}
+
\varepsilon_c(Q)
}{
n_{\mathrm{eff}}(w)
}
}.
\]
\end{corollary}

\begin{proof}
Fix $\pi\in\Pi$ and define
\[
\mu_i(\pi)
=
\mathbb E_{Z\sim\mathcal P_i}[\ell(Z,\pi)].
\]
Let
\[
X_i(\pi)
=
w_i\left(\mu_i(\pi)-\ell(Z_i,\pi)\right).
\]
Since $Z_i$ is independent of $Z_1,\ldots,Z_{i-1}$,
\[
\mathbb E[X_i(\pi)\mid\mathcal F_{i-1}]
=
w_i
\left(
\mu_i(\pi)
-
\mathbb E[\ell(Z_i,\pi)]
\right)
=
0.
\]
Thus $X_1(\pi),\ldots,X_n(\pi)$ are martingale differences.

Since $\ell(Z_i,\pi)\in[0,1]$,
\[
\mu_i(\pi)-\ell(Z_i,\pi)
\in
[\mu_i(\pi)-1,\mu_i(\pi)].
\]
Multiplying by $w_i\ge0$ gives
\[
X_i(\pi)
\in
[
w_i(\mu_i(\pi)-1),
w_i\mu_i(\pi)
].
\]
This interval may depend on $\pi$, but its length is always $w_i$. Therefore
the range-length parameter in
Theorem~\ref{thm:seldin_pac_bayes_azuma} is $r_i=w_i$, and
\[
\sum_{i=1}^n r_i^2
=
\sum_{i=1}^n w_i^2
=
\frac{1}{n_{\mathrm{eff}}(w)}.
\]
Finally,
\[
\sum_{i=1}^n X_i(\pi)
=
L_w(\pi)-\widehat L_w(\pi).
\]
Applying Theorem~\ref{thm:seldin_pac_bayes_azuma} and using the upper-tail
direction proves the result.
\end{proof}

\begin{theorem}[Exact Seldin-form weighted hitting-time bound]
\label{thm:exact_seldin_weighted_hitting_time}
Assume $0\le T_z^\pi\le H_{\max}$ for every task $z$ and policy $\pi$. Fix normalized BOOST weights $\tilde w_{\phi,1},\ldots,\tilde w_{\phi,n}$ over $\mathcal D_{\mathrm{mixed}}$, and let $\mathcal P_\phi$ be the weighted mixed task distribution defined by
\[
\mathbb E_{z\sim\mathcal P_\phi}[f(z)]
=
\sum_{i=1}^n
\tilde w_{\phi,i}\,
\mathbb E_{Z\sim\mathcal P_i}[f(Z)] .
\]
Let $P_0$ be the pretrained-policy prior chosen independently of the task-specific training trajectories, and let $Q_\theta$ be the posterior induced by the learned policy. Then, for any $1<c\le e$, with probability at least $1-\delta$,
\[
\begin{aligned}
\mathbb E_{\pi\sim Q_\theta}
\mathbb E_{z\sim\mathcal P_{\mathrm{task}}}
\left[
T_z^\pi
\right]
\le\;&
\mathbb E_{\pi\sim Q_\theta}
\left[
\sum_{i=1}^n
\tilde w_{\phi,i}\,T_{z_i}^{\pi}
\right] \\
&+
H_{\max}
\frac{1+c}{2\sqrt{2}}
\sqrt{
\frac{
\mathrm{KL}(Q_\theta\|P_0)
+
\log\frac{2}{\delta}
+
\varepsilon_c(Q_\theta)
}{
n_{\mathrm{eff}}(\tilde w_\phi)
}
}
+
\Delta_{\mathrm{task}}(\phi).
\end{aligned}
\]
\end{theorem}

\begin{proof}
Set $\ell(z,\pi) = T_z^\pi / H_{\max}$; by the bounded-horizon assumption, $\ell(z,\pi)\in[0,1]$. Applying Corollary~\ref{cor:weighted_pac_bayes_bounded} with $w_i=\tilde w_{\phi,i}$, $Q=Q_\theta$, and multiplying both sides by $H_{\max}$ yields
\[
\mathbb E_{\pi\sim Q_\theta}
\mathbb E_{z\sim\mathcal P_\phi}[T_z^\pi]
\le
\mathbb E_{\pi\sim Q_\theta}
\left[
\sum_{i=1}^n \tilde w_{\phi,i}\,T_{z_i}^\pi
\right]
+
H_{\max}
\frac{1+c}{2\sqrt{2}}
\sqrt{
\frac{
\mathrm{KL}(Q_\theta\|P_0)+\log\tfrac{2}{\delta}+\varepsilon_c(Q_\theta)
}{
n_{\mathrm{eff}}(\tilde w_\phi)
}
} .
\]
By definition of $\Delta_{\mathrm{task}}(\phi)$, for every $\pi\in\Pi$, $\mathbb E_{z\sim\mathcal P_{\mathrm{task}}}[T_z^\pi] \le \mathbb E_{z\sim\mathcal P_\phi}[T_z^\pi] + \Delta_{\mathrm{task}}(\phi)$. Taking expectation over $\pi\sim Q_\theta$ and combining the two inequalities proves the claim.
\end{proof}

\begin{corollary}[Simplified constant-form bound]
\label{cor:simplified_weighted_hitting_bound}
Under the same assumptions as
Theorem~\ref{thm:exact_seldin_weighted_hitting_time}, there exists a universal
numerical constant $C_1>0$ such that, for any $\delta\in(0,1)$, with
probability at least $1-\delta$,
\[
\begin{aligned}
\mathbb E_{\pi\sim Q_\theta}
\mathbb E_{z\sim\mathcal P_{\mathrm{task}}}
\left[
T_z^\pi
\right]
\le\;&
\mathbb E_{\pi\sim Q_\theta}
\left[
\sum_{i=1}^n
\tilde w_{\phi,i}T_{z_i}^{\pi}
\right] \\
&+
C_1H_{\max}
\sqrt{
\frac{
\mathrm{KL}(Q_\theta\|P_0)+\log(2/\delta)
}{
n_{\mathrm{eff}}(\tilde w_\phi)
}
}
+
\Delta_{\mathrm{task}}(\phi).
\end{aligned}
\]
\end{corollary}

\begin{proof}
Fix any $1<c\le e$ in
Theorem~\ref{thm:exact_seldin_weighted_hitting_time}. Since
$\log(1+x)\le x$ for $x\ge 0$, the term $\varepsilon_c(Q_\theta)$ is bounded by
a constant depending only on $c$ times
$\mathrm{KL}(Q_\theta\|P_0)+\log(2/\delta)$. Absorbing this constant and the
Hoeffding--Azuma prefactor $(1+c)/(2\sqrt{2})$ into $C_1$ proves the claim.
\end{proof}

\section{Hyperparameters}
In Table \ref{tab:hyperparams}, we provide the hyperparameter details for all our experiments.

\begin{table}[ht]
\centering
\caption{Hyperparameters for all experiments across three environments and two algorithms.}
\label{tab:hyperparams}
\renewcommand{\arraystretch}{1.35}
\setlength{\tabcolsep}{5pt}
\begin{tabular}{
  l
  l
  l
  c
  c
  c
  c
  c
  c
  c
  c
  c
}
\toprule
\multirow{2}{*}{\textbf{Environment}}
  & \multirow{2}{*}{\textbf{Algo.}}
  & \multirow{2}{*}{\textbf{Model}}
  & \rotatebox{90}{\textbf{LR}\hspace{4pt}}
  & \rotatebox{90}{$\boldsymbol{\alpha}$ \textbf{(init)}\hspace{4pt}}
  & \rotatebox{90}{\textbf{CQL weight}\hspace{4pt}}
  & \rotatebox{90}{\textbf{Batch size}\hspace{4pt}}
  & \rotatebox{90}{\textbf{Grad.\ accum.}\hspace{4pt}}
  & \rotatebox{90}{$K_\theta = K_\psi$\hspace{4pt}}
  & \rotatebox{90}{$K_\phi$\hspace{4pt}}
  & \rotatebox{90}{\textbf{Train frac}\hspace{4pt}}
  & \rotatebox{90}{\textbf{Val split}\hspace{4pt}}
  \\
\cmidrule(lr){4-12}
 
\rowcolor{tqbg}
Twenty Questions
  & MC
  & \texttt{gpt2}
  & 1e-4  & 0     & 10 & 8 & 16 & 20 & 1 & 0.6 & 0.3 \\
 
\rowcolor{tqbg}
Twenty Questions
  & ILQL
  & \texttt{gpt2}
  & 1e-4 & 1     & 10 & 8 & 16 & 20 & 1 & 0.6 & 0.3 \\
 
\rowcolor{carbg}
Car Dealer
  & ILQL
  & \texttt{gpt2}
  & 1e-4 & 1     & 10 & 8 & 2 & 32 & 1 & 0.6 & 0.3 \\
 
\rowcolor{carbg}
Car Dealer
  & MC
  & \texttt{gpt2}
  & 1e-4 & 10    & 10 & 8 & 2 & 32 & 1 & 0.6 & 0.3 \\
 
\rowcolor{webbg}
WebShop
  & MC
  & \texttt{LLaMA-3-8B}
  & 1e-4 & 0.1   & 10 & 4 & 4 & 32 & 1 & 0.6 & 0.3 \\
 
\rowcolor{webbg}
WebShop
  & ILQL
  & \texttt{LLaMA-3-8B}
  & 1e-4 & 0.1   & 10 & 1 & 16 & 32 & 1 & 0.6 & 0.3 \\
 
\bottomrule
\end{tabular}
\end{table}


\section{Task Categories and Category-Based Data Splits}
\label{app:categories}

A central design choice in our evaluation is to split data \emph{by task category} rather
than by individual task instance. This tests genuine generalisation: a model must transfer
its learned policy to semantic categories it has never optimised on, rather than simply
interpolating within a known distribution. Below we describe the categories for each
environment and the precise procedure used to construct the splits.

\subsection{Twenty Questions: Word Categories}
\label{app:categories:tq}

In the Twenty Questions environment each episode contains a hidden target word that the
agent must identify through yes/no questions. Every target word belongs to one of
17~semantic categories. The full category taxonomy with representative members is given in
Table~\ref{tab:tq_categories}.

\begin{table}[ht]
\centering
\caption{Twenty Questions word categories and sample objects.
The full item pool contains 10--14 objects per category.}
\label{tab:tq_categories}
\renewcommand{\arraystretch}{1.25}
\setlength{\tabcolsep}{6pt}
\begin{tabular}{l l}
\toprule
\textbf{Category} & \textbf{Sample objects} \\
\midrule
Sports              & Basketball, Football, Baseball, Tennis racket, Helmet \\
Animals             & Cat, Dog, Horse, Lion, Elephant \\
Fruits              & Apple, Banana, Strawberry, Watermelon, Mango \\
Vehicles            & Car, Motorcycle, Airplane, Helicopter, Ship \\
Clothes             & Shirt, Jacket, Dress, Boots, Scarf \\
Electronics         & Smartphone, Television, Camera, Refrigerator, Blender \\
Musical Instruments & Piano, Guitar, Violin, Trumpet, Harp \\
Furniture           & Chair, Bed, Couch, Bookcase, Nightstand \\
Office Supplies     & Pen, Stapler, Calculator, Scissors, Diary \\
Vegetables          & Carrot, Broccoli, Tomato, Spinach, Cucumber \\
Art                 & Painting, Paintbrush, Canvas, Sculpture, Marker \\
Kitchen Tools       & Knife, Fork, Bowl, Frying pan, Whisk \\
Nature              & Rock, Tree, Mountain, Ocean, Cactus \\
Toys                & Lego, Doll, Kite, Puzzle, Stuffed animal \\
Jewelry             & Earrings, Necklace, Ring, Watch, Pendant \\
Garden Supplies     & Shovel, Rake, Watering can, Lawn mower, Gloves \\
Tools               & Hammer, Screwdriver, Wrench, Saw, Drill \\
\bottomrule
\end{tabular}
\end{table}

\subsection{Car Dealer Negotiation: Brand Categories}
\label{app:categories:car}

In the Car Dealer Negotiation environment the buyer's preferred car brand encodes latent
information about price range, feature expectations, and negotiation style. We treat each
brand as a distinct task category. The 13~brands used in our experiments are:

\begin{center}
\textit{Volkswagen, Lexus, Ford, Mazda, Hyundai, Toyota,
Mercedes-Benz, BMW, Audi, Subaru, Honda, Porsche, Tesla.}
\end{center}

\subsection{WebShop Agent: Product Categories}
\label{app:categories:webshop}

Each WebShop episode specifies a purchase goal tied to a product drawn from Amazon's
category taxonomy. We use the full two-level category string as the task category label.
Our dataset spans 37~such categories, listed in Table~\ref{tab:webshop_categories}.

\begin{table}[ht]
\centering
\caption{WebShop product categories used as task categories (37 total).}
\label{tab:webshop_categories}
\renewcommand{\arraystretch}{1.2}
\setlength{\tabcolsep}{5pt}
\begin{tabular}{l l}
\toprule
\multicolumn{2}{l}{\textbf{Category (domain -- sub-category)}} \\
\midrule
Automotive -- Motorcycle \& Powersports
  & Baby Products -- Nursery \\
Beauty \& Personal Care -- Foot, Hand \& Nail Care
  & Beauty \& Personal Care -- Fragrance \\
Beauty \& Personal Care -- Hair Care
  & Beauty \& Personal Care -- Makeup \\
Beauty \& Personal Care -- Oral Care
  & Beauty \& Personal Care -- Personal Care \\
Beauty \& Personal Care -- Shave \& Hair Removal
  & Beauty \& Personal Care -- Skin Care \\
Beauty \& Personal Care -- Tools \& Accessories
  & Cell Phones \& Accessories -- Accessories \\
Cell Phones \& Accessories -- Cases, Holsters \& Sleeves
  & Clothing, Shoes \& Jewelry -- Men \\
Clothing, Shoes \& Jewelry -- Novelty \& More
  & Clothing, Shoes \& Jewelry -- Sport Specific Clothing \\
Clothing, Shoes \& Jewelry -- Women
  & Electronics -- Accessories \& Supplies \\
Electronics -- Camera \& Photo
  & Electronics -- Car \& Vehicle Electronics \\
Electronics -- Computers \& Accessories
  & Electronics -- Home Audio \\
Electronics -- Portable Audio \& Video
  & Grocery \& Gourmet Food -- Deli \& Prepared Foods \\
Grocery \& Gourmet Food -- Frozen
  & Grocery \& Gourmet Food -- Produce \\
Grocery \& Gourmet Food -- Snacks \& Sweets
  & Home \& Kitchen -- Bedding \\
Home \& Kitchen -- Furniture
  & Home \& Kitchen -- Home D\'{e}cor Products \\
Home \& Kitchen -- Kitchen \& Dining
  & Home \& Kitchen -- Wall Art \\
Office Products -- Office \& School Supplies
  & Office Products -- Office Furniture \& Lighting \\
Sports \& Outdoors -- Sports
  & Tools \& Home Improvement -- Lighting \& Ceiling Fans \\
\bottomrule
\end{tabular}
\end{table}

\subsection{Shared Split Protocol}
\label{app:categories:protocol}

Across all three environments the split procedure follows these steps:

\begin{enumerate}

  \item \textbf{Category enumeration.}
        All distinct category labels present in the full dataset are collected and fixed
        before any model training.

  \item \textbf{Train/eval category partition.}
        A fraction $f_{\mathrm{train}}$ of categories (controlled by
        \texttt{--train\_task\_frac}) is randomly assigned to the \emph{training} split.
        All remaining categories form the \emph{eval category pool}, which is shared
        between the validation and evaluation sets.
        All offline trajectories whose category falls in the training split are placed
        in $\mathcal{D}_{\mathrm{train}}$.

  \item \textbf{Val/eval trajectory partition within held-out categories.}
        For every category in the eval category pool, its trajectories are randomly
        divided into two groups.
        A fraction $f_{\mathrm{val}}$ of trajectories (controlled by
        \texttt{--val\_split}) are assigned to the \emph{validation} set
        $\mathcal{D}_{\mathrm{val}}$; the remaining $1 - f_{\mathrm{val}}$ fraction form
        the \emph{evaluation} set $\mathcal{D}_{\mathrm{eval}}$.
        Crucially, validation and evaluation trajectories are drawn from the
        \emph{same held-out categories}.

  \item \textbf{Synthetic data generation.}
        Synthetic trajectories are seeded exclusively from prefixes of
        $\mathcal{D}_{\mathrm{train}}$ trajectories.
        No synthetic data is generated from held-out categories, ensuring that
        $\mathcal{D}_{\mathrm{synth}}$ shares the category distribution of
        $\mathcal{D}_{\mathrm{train}}$ and leaks no information about held-out
        categories into training.

  \item \textbf{Evaluation.}
        At test time, the policy interacts with a live simulator on tasks drawn from the
        held-out categories exclusively, using the $\mathcal{D}_{\mathrm{eval}}$ partition.
        This constitutes a strictly out-of-distribution evaluation at the
        \emph{category level}.
    \item \textbf{Low data sample.}
        To simulate a low-data setting, we subsample 2.5\%, 10\%, and 30\% of the full
training data for Twenty Questions, Car Dealer, and WebShop respectively.

\end{enumerate}

The exact split fractions used per environment are summarised in
Table~\ref{tab:hyperparams}. The partition is fixed with a single random seed before
any model training and held constant across all methods for a given environment, ensuring
fair comparison.


\section{Synthetic Data Generation}
\label{app:synth}

For each training trajectory $\tau \in \mathcal{D}_{\mathrm{train}}$ we generate exactly
one synthetic counterpart, so $|\mathcal{D}_{\mathrm{synth}}| = |\mathcal{D}_{\mathrm{train}}|$.
To ensure a fair comparison with non-synthetic baselines, we always use the \emph{same}
LLM that will subsequently be fine-tuned as the data generator: GPT-2 for Twenty Questions
and Car Dealer Negotiation, and Llama-3-8B for WebShop.
This avoids any confound from generator capacity: synthetic baselines cannot benefit from
access to a stronger model, so performance differences are attributable solely to the
reweighting mechanism.

\subsection{Generation Procedure}

Each synthetic trajectory is constructed as follows.

\begin{enumerate}
  \item \textbf{Seeding from real data.}
        We take the first two timesteps of a real training trajectory, one environment
        observation and one agent action, as the prefix of the synthetic trajectory.
        This grounds the synthetic episode in the same initial conditions as the real data
        and ensures that synthetic tasks are drawn from the same category distribution as
        $\mathcal{D}_{\mathrm{train}}$.

  \item \textbf{Simulator-free self-play.}
        From the seed prefix onward, the LLM is prompted to play \emph{both} roles
        simultaneously: it generates the next agent action conditioned on the current
        history, then generates the resulting environment observation conditioned on that
        action, and so on.
        Crucially, this requires no access to the real environment simulator.
        At each step, the model receives (i) a system prompt describing the environment
        rules and action space, (ii) an example of a valid trajectory, and
        (iii) the full dialogue history so far.

  \item \textbf{Termination.}
        The episode ends when the model generates a terminal action (e.g.\
        \texttt{click[Buy Now]} in WebShop, a final guess in Twenty Questions, or an
        agreement or rejection in Car Dealer) or when the maximum episode length is
        reached.
\end{enumerate}

This procedure yields two important properties.
First, because the same model acts as both agent and environment, the synthetic
observations are consistent with the model's own internal world model, keeping
trajectories internally coherent without a simulator.
Second, because synthetic data is seeded from real prefixes, the early timesteps of each
synthetic trajectory remain grounded; distributional drift accumulates only in later turns,
where the bilevel reweighting head can learn to discount low-quality continuations.

\subsection{Environment Prompts}
\label{app:synth:prompts}

Below, in Figures \ref{fig:prompt_tq}, \ref{fig:prompt_car}, \ref{fig:prompt_webshop} we report the prompts used to condition the LLM in each role during synthetic data
generation for Twenty Questions, Car Dealer and Webshop Environments.

\label{app:synth:tq}

\begin{figure}[ht]
\begin{mdframed}
\small


\textbf{Agent main prompt}

\medskip
\texttt{%
Welcome to the game of Twenty Questions! Your objective is to guess what the
object is within twenty questions. At every turn, you will have the opportunity
to ask a yes/no question and receive an answer from the oracle. You can ask
twenty questions but must ask as few questions as possible.
}

\bigskip
\textbf{Environment system prompt (oracle)}

\medskip
\texttt{%
You are the oracle in a game of Twenty Questions. A hidden target object has
been chosen (provided in the goal). For every yes/no question the agent asks,
respond with exactly ``Yes.'' or ``No.'' based solely on whether the answer is
true of the target object. Do not reveal the object or provide any other
information. Output only your one-word answer.
}
\end{mdframed}
\caption{Twenty Questions prompts used during synthetic data generation.}
\label{fig:prompt_tq}
\end{figure}

\label{app:synth:car}

\begin{figure}[ht]
\begin{mdframed}
\small
\textbf{Agent system prompt (salesperson)}

\medskip
\texttt{%
You are an expert car salesperson conducting a price negotiation with a buyer.
Your goal is to reach a sale at the highest possible price while keeping the
buyer engaged. You may offer discounts and highlight vehicle features, but you
must not accept a price below your minimum threshold. Output only a single
dialogue turn.
}

\bigskip
\textbf{Environment system prompt (buyer)}

\medskip
\texttt{%
You are simulating a car buyer in a negotiation. Given the negotiation history,
generate the buyer's next response. The buyer has a target price and preferred
brand (provided in the goal). The buyer will counteroffer when the current
price exceeds their target, accept if it meets or falls below their target,
and walk away after a fixed number of failed rounds. Output only the buyer's
next utterance, nothing else.
}


\end{mdframed}
\caption{Car Dealer Negotiation prompts used during synthetic data generation.
Bracketed placeholders are filled from the seed trajectory's metadata.}
\label{fig:prompt_car}
\end{figure}

\label{app:synth:webshop}

\begin{figure}[ht]
\begin{mdframed}
\small
\textbf{Agent system prompt}

\medskip
\texttt{%
You are playing WebShop. You will see a text view of the current page and must
output exactly one action per turn.\\[4pt]
CRITICAL RULES:\\
-- search[query]: type a search query of 4--8 words maximum, extracting only
   core keywords (e.g., ``mens tuxedo shirt'').\\
-- You only get ONE search. Once you see search results, choose the best-matching
   product and use click[].\\
-- On a product page, click the relevant attribute options and then click[Buy Now].\\
-- If attributes or price do not match and a better item exists, click[\textless{} Prev]
   to return to results and try the next item.\\[4pt]
Output only one action per message in the form action[argument]. Nothing else.
}

\bigskip
\textbf{Environment simulator system prompt}

\medskip
\texttt{%
You are simulating the WebShop text-based shopping environment. Given the
trajectory so far, generate the NEXT environment observation.\\[4pt]
RULES:\\
1. After search[query]: return 3 product listings in the format\\
   \hspace*{1em}[ASIN] \textbackslash{}n Product Name \textbackslash{}n \$Price\\
   Include the target product (from goal info) as one of the 3 results.\\
2. After click[ASIN] on results: return the product detail page with option
   dropdowns, product name, price, rating, and [Buy Now]. Use real option values
   for the target product.\\
3. After click[option] on a product page: return
   ``You have clicked \{option\}.''\\
4. After click[\textless{} Prev]: return the most recent search results page.\\[4pt]
Output ONLY the observation text. Match exact WebShop formatting.
}
\end{mdframed}
\caption{WebShop prompts used during synthetic data generation.}
\label{fig:prompt_webshop}
\end{figure}

In all three environments, a representative example trajectory (formatted identically to
the real data) is appended to the respective system prompts before generation begins, to
disambiguate the expected output format and action grammar.

\section{Training Implementation Details}
\label{app:impl}
 
\subsection{Base Language Models}
 
For Twenty Questions and Car Dealer Negotiation we fine-tune \textbf{GPT-2}
(117M parameters), which has been shown to perform adequately on these tasks in
prior work~\cite{abdulhai2025lmrl}.
For WebShop, which requires multi-step information gathering and diverse action types
across a large product space, we fine-tune \textbf{Llama-3-8B} (8B parameters).
To make fine-tuning tractable at this scale, we apply \textbf{QLoRA}~\cite{dettmers2023qlora}:
the base model weights are quantised to 4-bit precision and kept frozen, while
low-rank adapter matrices are trained in 16-bit precision.
This reduces the memory footprint of the base model significantly.
 
\subsection{Q and V Head Architecture}
 
BOOST augments the base language model with lightweight value heads that are trained
on top of base LLMs final hidden layer outputs.
The head architecture differs by RL algorithm:
 
\begin{itemize}
  \item \textbf{Monte Carlo (MC).} A single Q-head maps each token's hidden state to a
        scalar return estimate.
        The head is a linear projection:
        \[
          Q_\theta(s, a) \;=\; W_Q \,h \;+\; b_Q,
          \quad W_Q \in \mathbb{R}^{|\mathcal{V}| \times d},
        \]
        where $h \in \mathbb{R}^d$ is the final hidden state of the LM (with $d$ the
        model embedding dimension) and $|\mathcal{V}|$ is the vocabulary size.
 
  \item \textbf{ILQL.} Two Q-heads (for the double-Q trick) and one V-head are used.
        The Q-heads share the same architecture as in MC above; the V-head is a scalar
        projection:
        \[
          V_\theta(s) \;=\; W_V \,h \;+\; b_V,
          \quad W_V \in \mathbb{R}^{1 \times d}.
        \]
\end{itemize}
 
Both head types are thus \emph{depth-1 linear models} with no non-linearity, keeping
the number of additional parameters negligible relative to the base LM.

\subsection{Memory-Efficient Auxiliary Model}
 
BOOST maintains an auxiliary policy $\psi$ that tracks the inner-level optimum (see
Algorithm~1). Instead of storing a full second copy of all model
parameters, doubling memory consumption our
auxiliary policy $\psi$ and the primary policy $\theta$ differ only in their value heads:
the base LM backbone is shared between the two.
Concretely, we keep a \emph{single copy} of the base LM weights in GPU memory and
maintain \emph{separate copies of only the Q-heads} (for MC) or \emph{the Q-heads and
V-head} (for ILQL) for $\theta$ and $\psi$ respectively.
Because the value heads are depth-1 linear layers, the additional memory cost of the
auxiliary model is a negligible fraction of the total.
 
\subsection{Reweighting Head $w_\phi$}
 
The reweighting head $w_\phi$ takes as input an embedding representation of a trajectory
and outputs a scalar weight.
We use a fully connected MLP with ReLU activations:
 
\begin{itemize}
  \item \textbf{Twenty Questions and Car Dealer Negotiation:} depth-2 MLP
        (one hidden layer).
  \item \textbf{WebShop:} depth-4 MLP (three hidden layers), reflecting the greater
        complexity and diversity of the WebShop action space.
\end{itemize}
 
In all cases the MLP is small relative to the base model, adding negligible compute and
memory overhead to the bilevel updates.
 
\subsection{Trajectory Embeddings}
\label{app:impl:embeddings}
 
To obtain a fixed-dimensional representation of each trajectory for input to $w_\phi$,
we pass the trajectory through the \emph{pretrained} (unmodified) base language model and
extract the final-layer hidden state of the last token as the trajectory
embedding. This embedding is computed once per trajectory before training begins and
cached, so it does not contribute to the per-iteration training cost.
Using the pretrained model's representations rather than the continuously updating
fine-tuned model ensures that the embedding space remains stable throughout training,
preventing the input distribution to $w_\phi$ from shifting as $\theta$ evolves.

\section{Ablation of Stabilizing Design Choices}
\label{app:ablation}

\begin{table}[ht]
\centering
\caption{$2\times2\times2$ ablation of the design choices in Section~\ref{sec:stable-updates} (WebShop, MC, low-data). Reward is mean $\pm$ standard error. $g_{\text{rel}}$ is the relative loss gap between the primary policy $\theta$ and the auxiliary policy $\psi$ (lower is better).}
\label{tab:ablation}
\setlength{\tabcolsep}{5pt}
\small
\begin{tabular}{c ccc cc}
\toprule
\textbf{Arm} & \textbf{Phased updates} & \textbf{Warm-start} ($\theta \leftarrow \psi$) & \textbf{Slow LR} ($\phi$) & \textbf{Reward} $\uparrow$ & $g_{\text{rel}}$ $\downarrow$ \\
\midrule
1 & \checkmark & \checkmark & \checkmark & $\mathbf{0.433 \pm 0.031}$ & $\mathbf{0.005}$ \\
2 & \checkmark & \checkmark & --         & $0.405 \pm 0.034$ & $0.012$ \\
3 & \checkmark & --         & \checkmark & $0.269 \pm 0.048$ & $0.445$ \\
4 & \checkmark & --         & --         & $0.209 \pm 0.081$ & $0.463$ \\
5 & --         & \checkmark & \checkmark & $0.207 \pm 0.145$ & $0.106$ \\
6 & --         & \checkmark & --         & $0.365 \pm 0.045$ & $0.096$ \\
7 & --         & --         & \checkmark & $0.233 \pm 0.046$ & $1.941$ \\
8 & --         & --         & --         & $0.196 \pm 0.041$ & $1.895$ \\
\bottomrule
\end{tabular}
\end{table}

To assess how much each design choice in Section~\ref{sec:stable-updates} contributes to stability, we run a full $2\times2\times2$ ablation over: (i) phased (alternating) updates of $\psi$, $\theta$, and $\phi$; (ii) warm-starting the primary policy from the auxiliary policy ($\theta \leftarrow \psi$) at each outer iteration; and (iii) a slow learning rate for the reweighting head $\phi$ (asymmetric learning rates).

We define the relative loss gap as $g_{\text{rel}} = |\mathcal{L}_{\text{train}}(\theta,\phi) - \mathcal{L}_{\text{train}}(\psi,\phi)| \,/\, \mathcal{L}_{\text{train}}(\psi,\phi)$, which measures how closely the primary policy $\theta$ tracks the inner-level optimum $\psi$. A small $g_{\text{rel}}$ is critical for the bilevel formulation to be a good approximation, since we rely on the assumption $\mathcal{L}_{\text{train}}(\theta,\phi) \approx \mathcal{L}_{\text{train}}(\psi,\phi)$, i.e., the primary policy stays close to the inner-level optimum under the learned weights. Across Arms 1--8 (Table~\ref{tab:ablation}), we observe:

\begin{itemize}[leftmargin=*]
    \item \textbf{Phased updates and warm-start each independently help.} Turning on either one reduces the loss gap relative to the fully-off configuration (Arm 8). When both are on (Arms 1--2), the gap shrinks by more than two orders of magnitude ($0.005$--$0.012$ vs.\ $1.895$) and reward is highest, showing that both choices are needed for bilevel optimization to behave as intended in practice .
    \item \textbf{The slow learning rate for $\phi$ plays a secondary role.} It is most useful together with phased updates (Arm 1 vs.\ 2: $0.433$ vs.\ $0.405$, with the gap halved. In settings without other stabilising tricks ( Arm 7 vs 8), it modestly improves reward while keeping the gap from further exploding.
\end{itemize}

These results show that each design choice contributes to stable and effective trajectory-level bilevel reweighting.

\section{Compute Infrastructure}
We run all our experiments on Nvidia A100 GPUs with 80 GB memory. Each experiment on Twenty Question environment runs within 14 hours while experiments on Car Dealer Environment and Webshop environment run within 10 hours.

\section{Qualitative Analysis of Synthetic Data}
\label{app:qualitative}

Table~\ref{tab:traj_comparison} provides a side-by-side comparison of WebShop synthetic trajectories assigned high and low weights by $w_\phi$. The high-weight trajectory follows the task instruction faithfully, executes actions that are grounded in environment observations, and completes a successful purchase. The low-weight 
trajectory, by contrast, exhibits two characteristic failure patterns: hallucinated 
action targets, clicking on a product identifier that never appeared in the 
environment's output, and environment-rule violations in the form of a forbidden 
repeated search, both of which derail task completion before a reward can be obtained. 
That $w_\phi$ learns to distinguish these cases without explicit supervision on 
action validity is noteworthy: the reweighting signal comes purely from downstream 
validation reward, yet the learned weights recover a semantically meaningful notion 
of trajectory quality. This qualitative evidence is consistent with our quantitative 
findings in Figure \ref{fig:weight_all} that BOOST assigns higher weights to synthetic transitions 
that most closely resemble valid, executable offline behaviour.

\begin{table}[tb]
\centering
\caption{Side-by-Side Comparison of Good vs. Bad WebShop Trajectories}
\label{tab:traj_comparison}
\renewcommand{\arraystretch}{1.25}
\begin{tabularx}{\textwidth}{@{}p{1.4cm} X X@{}}
\toprule
\textbf{Phase / Key Timestep} 
& \textbf{Good Synthetic Sample (\%weight shift: 5.2\%)} 
& \textbf{Bad Synthetic Sample \% weight shift:-2.8\%} \\
\midrule

\textbf{T0: Init} 
& \textbf{State:} Environment reset; instruction specifies long-sleeve men’s T-shirt, color \emph{green}, size \emph{3x-large big}, price $<\$70$.  
  \newline \textbf{Action:} (implicit) \texttt{reset}  
  \newline \textbf{Reward:} 0.0  
& \textbf{State:} Environment reset; instruction with dress shirt, specific fabrics, color \emph{kelly green}, fit \emph{men}, size \emph{2T}, price $<\$50$.  
  \newline \textbf{Action:} (implicit) \texttt{reset}  
  \newline \textbf{Reward:} 0.0  
  \\
\midrule

\textbf{T1: Search} 
& \textbf{State:} Main webpage with instruction.  
  \newline \textbf{Action:} \texttt{search[machine wash men's t-shirts long sleeve green 3x-large]}  
  \newline \textbf{Reward:} 0.0  
  \newline \textbf{Comment:} Single query with key attributes (category, sleeve, color, size). 
& \textbf{State:} Main webpage with instruction.  
  \newline \textbf{Action:} \texttt{search[mens dress shirts kelly green 2t]}  
  \newline \textbf{Reward:} 0.0  
  \newline \textbf{Comment:} single query; aligned with the goal at this point. 
  \\
\midrule

\textbf{T2: First Product Choice} 
& \textbf{State:} Search results show multiple long-sleeve T-shirts, including \texttt{B00O30JLDK} with relevant attributes.  
  \newline \textbf{Action:} \texttt{click[B00O30JLDK]}  
  \newline \textbf{Reward:} 0.0  
  \newline \textbf{Comment:} Chooses a \emph{valid} ID from the current results (no hallucination). 
& \textbf{State:} Search results show IDs \texttt{B0969G2DH8}, \texttt{B07F9K8K3R}, \texttt{B07R9T8K3R}.  
  \newline \textbf{Action:} \textcolor{red}{\texttt{click[B09QQP3356]}}  
  \newline \textbf{Reward:} 0.0  
  \newline \textbf{Comment:} Hallucinates an ID not in results, causing jump to unrelated product. \\
\midrule

\textbf{T3: Refinement vs. Drift} 
& \textbf{State:} On product page \texttt{B00O30JLDK}; size options include \emph{3x-large big}, color options include \emph{green}.  
  \newline \textbf{Action:} \texttt{click[3x-large big]}  
  \newline \textbf{Reward:} 0.0  
  \newline \textbf{Comment:} Refines selection by choosing the exact requested size; still on a good candidate item. 
& \textbf{State:} On a polo shirt page (\texttt{B09QQP3356}) with wrong category (polo, short sleeve), wrong colors, no 2T.  
  \newline \textbf{Action:} \texttt{think[no valid action]} and then attempts a second search later.  
  \newline \textbf{Reward:} 0.0  
  \newline \textbf{Comment:} Recognizes difficulty but fails to recover (should navigate back or pick a valid ID from original results). \\
\midrule

\textbf{T4: Near-Term Goal Alignment} 
& \textbf{State:} Same product page, now with size set correctly; price \$49.99 $<\$70$, color \emph{green} available.  
  \newline \textbf{Action:} \texttt{click[Buy Now]}  
  \newline \textbf{Reward:} \textbf{1.0}  
  \newline \textbf{Comment:} All constraints (type, sleeve, color, size, price) satisfied at checkout. 
& \textbf{State:} After bouncing between unrelated items and invalid actions, no product matches color \emph{kelly green} and size \emph{2T}.  
  \newline \textbf{Action:} \texttt{click[Buy Now]} from a mismatched context.  
  \newline \textbf{Reward:} \textbf{0.0}  
  \newline \textbf{Comment:} Attempts to check out without satisfying core constraints; failure stems from early hallucination plus rule violations and lack of recovery. \\
\bottomrule
\end{tabularx}
\end{table}

\clearpage
\newpage
\section*{NeurIPS Paper Checklist}

\begin{enumerate}

\item {\bf Claims}
    \item[] Question: Do the main claims made in the abstract and introduction accurately reflect the paper's contributions and scope?
    \item[] Answer: \answerYes{}
\item[] Justification: The abstract and introduction clearly state the three main
contributions:trajectory level bilevel reweighting with synthetic data, PAC-Bayesian Analysis of generalization performance, and empirical improvements and analysis of results.

\item {\bf Limitations}
    \item[] Question: Does the paper discuss the limitations of the work performed by the authors?
    \item[] Answer: \answerYes{} 
    \item[] Justification: We briefly mention the limitation of our work in the conclusion, specifying that we only increase reweighting granularity upto trajectory level and a natural extension would be to learn weights at per dialogue level.

\item {\bf Theory assumptions and proofs}
    \item[] Question: For each theoretical result, does the paper provide the full set of assumptions and a complete (and correct) proof?
    \item[] Answer: \answerYes{}, 
    \item[] Justification: We provide a thorough set of assumptions of our theoretical results and proof, particularly the setting of finite horizon setting .

    \item {\bf Experimental result reproducibility}
    \item[] Question: Does the paper fully disclose all the information needed to reproduce the main experimental results of the paper to the extent that it affects the main claims and/or conclusions of the paper (regardless of whether the code and data are provided or not)?
    \item[] Answer: \answerYes{} 
    \item[] Justification: In the appendix, we give complete details of hyperparameters used in all our experiment results. We also give details on the BOOST algorithm framework along with the training setup, data generations setup and prompts used for synthetic data generation. 

\item {\bf Open access to data and code}
    \item[] Question: Does the paper provide open access to the data and code, with sufficient instructions to faithfully reproduce the main experimental results, as described in supplemental material?
    \item[] Answer: \answerYes{} 
    \item[] Justification: We provide a zip file of the codebase for running the experiments.

\item {\bf Experimental setting/details}
    \item[] Question: Does the paper specify all the training and test details (e.g., data splits, hyperparameters, how they were chosen, type of optimizer) necessary to understand the results?
    \item[] Answer: \answerYes{} 
    \item[] Justification: We provide the core details of experimental setup int he main paper. In the appendix, we give complete set of hyperparameters as well as data generation and data split details.

\item {\bf Experiment statistical significance}
    \item[] Question: Does the paper report error bars suitably and correctly defined or other appropriate information about the statistical significance of the experiments?
    \item[] Answer: \answerYes{} 
    \item[] Justification: Yes, we provide both mean and standard error in the main results.

\item {\bf Experiments compute resources}
    \item[] Question: For each experiment, does the paper provide sufficient information on the computer resources (type of compute workers, memory, time of execution) needed to reproduce the experiments?
    \item[] Answer:\answerYes{} 
    \item[] Justification: In the appendix, we provide details on the compute infrastructure used and the summary of compute time.
    
\item {\bf Code of ethics}
    \item[] Question: Does the research conducted in the paper conform, in every respect, with the NeurIPS Code of Ethics \url{https://neurips.cc/public/EthicsGuidelines}?
    \item[] Answer: \answerYes{} 
    \item[] Justification: Yes, the research adheres to the Code od Ethics.

\item {\bf Broader impacts}
    \item[] Question: Does the paper discuss both potential positive societal impacts and negative societal impacts of the work performed?
    \item[] Answer: \answerNA{}
    \item[] Justification: This work is situated in a broad class of works on fine-tuning LLMs and the framework proposed does not directly have any Societal impact.
    
\item {\bf Safeguards}
    \item[] Question: Does the paper describe safeguards that have been put in place for responsible release of data or models that have a high risk for misuse (e.g., pre-trained language models, image generators, or scraped datasets)?
    \item[] Answer: \answerNA{} 
    \item[] Justification: No new models or datasets are released.

\item {\bf Licenses for existing assets}
    \item[] Question: Are the creators or original owners of assets (e.g., code, data, models), used in the paper, properly credited and are the license and terms of use explicitly mentioned and properly respected?
    \item[] Answer: \answerNA{} 
    \item[] Justification: No licenced assets are used. Prior works are cited wherever applicable.

\item {\bf New assets}
    \item[] Question: Are new assets introduced in the paper well documented and is the documentation provided alongside the assets?
    \item[] Answer: \answerNA{} 
    \item[] Justification: No new assets are released in this paper.

\item {\bf Crowdsourcing and research with human subjects}
    \item[] Question: For crowdsourcing experiments and research with human subjects, does the paper include the full text of instructions given to participants and screenshots, if applicable, as well as details about compensation (if any)? 
    \item[] Answer: \answerNA{} 
    \item[] Justification: The paper does not involve research with human subjects or any crowdsourcing.

\item {\bf Institutional review board (IRB) approvals or equivalent for research with human subjects}
    \item[] Question: Does the paper describe potential risks incurred by study participants, whether such risks were disclosed to the subjects, and whether Institutional Review Board (IRB) approvals (or an equivalent approval/review based on the requirements of your country or institution) were obtained?
    \item[] Answer: \answerNA{} 
    \item[] Justification: The paper does not involve research with human subjects or any crowdsourcing.

\item {\bf Declaration of LLM usage}
    \item[] Question: Does the paper describe the usage of LLMs if it is an important, original, or non-standard component of the core methods in this research? Note that if the LLM is used only for writing, editing, or formatting purposes and does \emph{not} impact the core methodology, scientific rigor, or originality of the research, declaration is not required.
    \item[] Answer: \answerNA{} 
    \item[] Justification: LLM is not used in core development in this paper.

\end{enumerate}

\end{document}